\documentclass[a4paper,11pt]{article}

\usepackage{times}
\usepackage{amsmath}
\usepackage{algorithmic}
\usepackage{algorithm}
\usepackage{amsthm}
\usepackage{amsfonts}
\usepackage{comment}
\usepackage{graphicx}
\usepackage{hyperref,fullpage}
\usepackage{amsmath,amsfonts,amsthm,amssymb}
\usepackage{color}
\def\shmak{{\it C\ }}
\newcommand{\SRC}{\mathcal{SRW}}
\newcommand{\polyring}[1]{\mathbb{R}[x]_{#1}}
\newcommand{\Po}{P_{\Omega}}

\newcommand{\gNorm}[1]{\left|#1\right|_{1}}
\newcommand{\sos}{\textup{sos}}

\newcommand{\ggNorm}[1]{{\left\vert\kern-0.25ex\left\vert\kern-0.25ex\left\vert #1 		\right\vert\kern-0.25ex\right\vert\kern-0.25ex\right\vert}}

\newcommand{\mper}{\,.}

\newcommand{\R}{\mathbb{R}}

\newcommand{\Var}{\textup{Var}}

\newcommand{\Exp}{\mathop{\mathbb{E}}}

\newcommand{\mat}{\mathcal{M}}

\newcommand{\pca}{\textup{pca}}
\newcommand{\spd}{\textup{sa}}

\newcommand{\cH}{\mathcal{H}}
\newcommand{\tp}[1]{^{\otimes #1}}
\newcommand{\op}{\textup{op}}
\newcommand{\tsop}{\textup{op}}
\newcommand{\cD}{\D}

\newcommand{\inner}[1]{\langle #1 \rangle }
\newcommand{\Norm}[1]{\left\|#1\right\|}

\newcommand{\pE}{\tilde{\mathbb{E}}}

\usepackage[american]{babel}

\usepackage{amsthm}

\usepackage{algorithm}
\usepackage{algorithmic}

\newtheorem*{fact*}{Fact}
\newtheorem*{hypothesis*}{Hypothesis}

\theoremstyle{remark}

\newtheorem*{claim*}{Claim}

\newtheorem*{remark*}{Remark}

\newtheorem*{observation*}{Observation}

\usepackage{textcomp} 
\usepackage[scr=rsfso]{mathalfa}
\usepackage{bm} 

\usepackage{prettyref}

\newcommand{\savehyperref}[2]{\texorpdfstring{\hyperref[#1]{#2}}{#2}}

\newrefformat{eq}{\savehyperref{#1}{\textup{(\ref*{#1})}}}
\newrefformat{eqn}{\savehyperref{#1}{\textup{(\ref*{#1})}}}
\newrefformat{lem}{\savehyperref{#1}{Lemma~\ref*{#1}}}
\newrefformat{def}{\savehyperref{#1}{Definition~\ref*{#1}}}
\newrefformat{thm}{\savehyperref{#1}{Theorem~\ref*{#1}}}
\newrefformat{cor}{\savehyperref{#1}{Corollary~\ref*{#1}}}
\newrefformat{cha}{\savehyperref{#1}{Chapter~\ref*{#1}}}
\newrefformat{sec}{\savehyperref{#1}{Section~\ref*{#1}}}
\newrefformat{app}{\savehyperref{#1}{Appendix~\ref*{#1}}}
\newrefformat{tab}{\savehyperref{#1}{Table~\ref*{#1}}}
\newrefformat{fig}{\savehyperref{#1}{Figure~\ref*{#1}}}
\newrefformat{hyp}{\savehyperref{#1}{Hypothesis~\ref*{#1}}}
\newrefformat{alg}{\savehyperref{#1}{Algorithm~\ref*{#1}}}
\newrefformat{rem}{\savehyperref{#1}{Remark~\ref*{#1}}}
\newrefformat{item}{\savehyperref{#1}{Item~\ref*{#1}}}
\newrefformat{step}{\savehyperref{#1}{step~\ref*{#1}}}
\newrefformat{conj}{\savehyperref{#1}{Conjecture~\ref*{#1}}}
\newrefformat{fact}{\savehyperref{#1}{Fact~\ref*{#1}}}
\newrefformat{prop}{\savehyperref{#1}{Proposition~\ref*{#1}}}
\newrefformat{prob}{\savehyperref{#1}{Problem~\ref*{#1}}}
\newrefformat{claim}{\savehyperref{#1}{Claim~\ref*{#1}}}
\newrefformat{relax}{\savehyperref{#1}{Relaxation~\ref*{#1}}}
\newrefformat{red}{\savehyperref{#1}{Reduction~\ref*{#1}}}
\newrefformat{part}{\savehyperref{#1}{Part~\ref*{#1}}}

\newcommand{\Sref}[1]{\hyperref[#1]{\S\ref*{#1}}}

\usepackage{nicefrac}

\newcommand{\Paren}[1]{\left(#1\right)}

\newcommand{\Set}[1]{\left\{#1\right\}}

\newcommand{\Psymb}{\mathbb{P}}

\DeclareMathOperator*{\ProbOp}{\Psymb}

\renewcommand{\Pr}{\ProbOp}

\newcommand{\textparen}[1]{\text{(#1)}}

\ifx\because\undefined
\newcommand{\because}[1]{\textparen{because #1}}
\else
\renewcommand{\because}[1]{\textparen{because #1}}
\fi

\newcommand{\mcom}{\,,}

\newcommand\bdot\bullet

\ifx\mathds\undefined 

\else
}
\fi

\DeclareMathOperator{\argmax}{argmax}

\newcommand{\cA}{\mathcal A}

\newcommand{\cM}{\mathcal M}

\renewcommand{\leq}{\leqslant}
\renewcommand{\le}{\leqslant}

\renewcommand{\ge}{\geqslant}

\let\epsilon=\varepsilon

\numberwithin{equation}{section}

\newcommand\MYcurrentlabel{xxx}

\newcommand{\MYstore}[2]{%
  \global\expandafter \def \csname MYMEMORY #1 \endcsname{#2}%
}

\newcommand{\MYload}[1]{%
  \csname MYMEMORY #1 \endcsname%
}

\newcommand{\MYnewlabel}[1]{%
  \renewcommand\MYcurrentlabel{#1}%
  \MYoldlabel{#1}%
}

\newcommand{\MYdummylabel}[1]{}

\newcommand{\torestate}[1]{%
  \let\MYoldlabel\label%
  \let\label\MYnewlabel%
  #1%
  \MYstore{\MYcurrentlabel}{#1}%
  \let\label\MYoldlabel%
}

\newcommand{\restatetheorem}[1]{%
  \let\MYoldlabel\label
  \let\label\MYdummylabel
  \begin{theorem*}[Restatement of \prettyref{#1}]
    \MYload{#1}
  \end{theorem*}
  \let\label\MYoldlabel
}

\newcommand{\restatelemma}[1]{%
  \let\MYoldlabel\label
  \let\label\MYdummylabel
  \begin{lemma*}[Restatement of \prettyref{#1}]
    \MYload{#1}
  \end{lemma*}
  \let\label\MYoldlabel
}

\newcommand{\restateprop}[1]{%
  \let\MYoldlabel\label
  \let\label\MYdummylabel
  \begin{proposition*}[Restatement of \prettyref{#1}]
    \MYload{#1}
  \end{proposition*}
  \let\label\MYoldlabel
}

\newcommand{\restatefact}[1]{%
  \let\MYoldlabel\label
  \let\label\MYdummylabel
  \begin{fact*}[Restatement of \prettyref{#1}]
    \MYload{#1}
  \end{fact*}
  \let\label\MYoldlabel
}

\newcommand{\restate}[1]{%
  \let\MYoldlabel\label
  \let\label\MYdummylabel
  \MYload{#1}
  \let\label\MYoldlabel
}

\let\origparagraph\paragraph
\renewcommand{\paragraph}[1]{\origparagraph{#1.}}

\allowdisplaybreaks

\usepackage{paralist}

\usepackage{comment}

\newtheorem{theorem}{Theorem}[section]
\newtheorem{definition}{Definition}[section]

\newtheorem{proposition}[theorem]{Proposition}
\newtheorem{corollary}[theorem]{Corollary}
\newtheorem{lemma}[theorem]{Lemma}

\def\Y{{\mathcal Y}}

\newcommand{\D}{{\mathcal D}}
\renewcommand{\H}{{\mathcal H}}
\newcommand{\X}{{\mathcal X}}

\def\L{{\mathcal L}}

\def\reals{{\mathbb R}}
\def\R{{\mathbb R}}

\def\norm#1{\mathopen\| #1 \mathclose\|}

\newcommand{\conv}{\mathop{\mbox{\rm conv}}}

\newcommand{\poly}{\mathop{\mbox{\rm poly}}}

\newcommand{\loss}{\mathop{\mbox{\rm loss}}}
\newcommand{\E} {\mathop{\mathbb E}}

\DeclareMathOperator*{\argmin}{arg\,min}

\newcommand{\ignore}[1]{}

\def\trace{{\bf Tr}}

\def\reals{{\mathbb R}}

\def\bold0{\mathbf{0}}

\def\trace{{\bf Tr}}
\def\Exp{\mathop{\mathbb E}}

\renewcommand{\deg}{\mathop{\mbox{\rm deg}}}
\newcommand{\eps}{\varepsilon}

\newcommand{\x}{\ensuremath{\mathbf x}}

\def\eps{\varepsilon}
\def\epsilon{\varepsilon}

\def\R{\ensuremath{\mathbb R}}

\usepackage{times}

\title{A Non-generative Framework and Convex Relaxations for Unsupervised Learning}
\author{Elad Hazan \and Tengyu Ma}

\def\shownotes{0} 
 \ifnum\shownotes=1
\newcommand{\authnote}[2]{$\ll$\textsf{\footnotesize #1 notes: #2}$\gg$}
\else
\newcommand{\authnote}[2]{}
\fi

\begin{document}
\maketitle

\begin{abstract}
We give a novel formal theoretical framework for unsupervised learning with two distinctive characteristics. First, it does not assume any generative model and based on a worst-case performance metric. Second, it is comparative, namely performance is measured with respect to a given hypothesis class. This allows to avoid known computational hardness results and improper algorithms based on convex relaxations.  We show how several families of unsupervised learning models, which were previously only analyzed under probabilistic assumptions and are otherwise provably intractable, can be efficiently learned in our framework by convex optimization.  
\end{abstract}


\section{Introduction}

Unsupervised learning is the task of learning structure from unlabelled examples.  Informally, the main goal of unsupervised learning is to extract structure from the data in a way that will enable efficient learning from future labelled examples for potentially numerous independent tasks.  

It is useful to recall the Probably Approximately Correct (PAC) learning theory for supervised learning \cite{valiant1984theory}, based on Vapnik's statistical learning theory \cite{Vapnik1998}. In PAC learning, the learning can access labelled examples from an unknown distribution. On the basis of these examples, the learner constructs a hypothesis that generalizes to unseen data.  A concept is said to be learnable with respect to a hypothesis class if there exists an (efficient) algorithm that outputs a generalizing hypothesis with high probability after observing polynomially many examples in terms of the input representation. 

The great achievements of PAC learning that made it successful are its generality and algorithmic applicability:   PAC learning does not restrict the input domain in any way, and thus allows very general learning, without generative or distributional assumptions on the world.  Another important feature is the restriction to specific hypothesis classes,  without which there are simple impossibility results such as the ``no free lunch" theorem. This allows {\it comparative} and {\it improper} learning of  computationally-hard concepts.

The latter is a very important point which is often understated. Consider the example of sparse regression, which is a canonical problem in high dimensional statistics. 
Fitting the best sparse vector to linear prediction is an NP-hard problem~\cite{DBLP:journals/siamcomp/Natarajan95}. 
However, this does not prohibit improper  learning, since we can use a $\ell_1$ convex relaxation for the sparse vectors (famously known as LASSO~\cite{10.2307/2346178}).

Unsupervised learning, on the other hand, while extremely applicative and well-studied, has not seen such an inclusive theory. The most common approaches, such as restricted Boltzmann machines, topic models, dictionary learning, principal component analysis and metric clustering,  are based almost entirely on generative assumptions about the world.  This is a strong restriction which makes it very hard to analyze such approaches in scenarios for which the assumptions do not hold.   A more discriminative approach is based on compression, such as the Minimum Description Length criterion. This approach gives rise to provably intractable problems and doesn't allow improper learning.

\paragraph{Main results} 

We start by proposing a rigorous framework for unsupervised learning which allows data-dependent, comparative learning without  generative assumptions.  It is general enough to encompass previous methods such as PCA, dictionary learning and topic models. Our main contribution are optimization-based relaxations and efficient algorithms that are shown to improperly probably learn previous models, specifically: 
\begin{enumerate}
\item
We consider the classes of hypothesis known as dictionary learning. We give a more general hypothesis class which encompasses and generalizes it according to our definitions. We proceed to give novel polynomial-time algorithms for learning the broader class. These algorithms are based on new techniques in unsupervised learning, namely sum-of-squares convex relaxations. 

As far as we know, this is the first result for efficient improper learning of dictionaries without generative assumptions. Moreover, our result handles polynomially over-complete dictionaries, while previous works~\cite{DBLP:conf/colt/AroraGMM15,DBLP:conf/stoc/BarakKS15} apply to at most constant factor over-completeness.

\item
We give efficient algorithms for learning a new hypothesis class which we call spectral autoencoders. We show that this class generalizes, according to our definitions,  the class of PCA (principal component analysis) and its kernel extensions.

\end{enumerate}

\paragraph{Structure of this paper} In the following chapter we a non-generative, distribution-dependent definition for unsupervised learning which mirrors that of PAC learning for supervised learning. 
We then proceed to an illustrative example and show how Principal Component Analysis can be formally learned in this setting.  The same section also gives a much more general class of hypothesis for unsupervised learning which we call polynomial spectral decoding, and show how they can be efficient learned in our framework using convex optimization. Finally, we get to our main contribution: a convex optimization based methodology for improper learning a wide class of hypothesis, including dictionary learning.

\subsection{Previous work}

The vast majority of work on unsupervised learning, both theoretical as well as applicative,  focuses on generative models. These include topic models \cite{Blei2003}, dictionary learning \cite{Donoho2006},  Deep Boltzmann Machines and deep belief networks \cite{Salakhutdinov2009} and many more.  Many times these models entail non-convex optimization problems that are provably NP-hard to solve in the worst-case. 

A recent line of work in theoretical machine learning attempts to give efficient algorithms for these models with provable guarantees. Such algorithms were given for topic models \cite{arora2012learning}, dictionary learning \cite{arora2013new,DBLP:conf/colt/AroraGMM15}, mixtures of gaussians and hidden Markov models \cite{hsu2013learning,Anandkumar2014} and more.   However, these works retain, and at times even enhance, the probabilistic generative assumptions of the underlying model.  
Perhaps the most widely used unsupervised learning methods are clustering algorithms such as k-means, k-medians and principal component analysis (PCA), though these  lack generalization guarantees.  An axiomatic approach to clustering was initiated by Kleinberg \cite{kleinberg03} and pursued further in \cite{benda09}.  A discriminative generalization-based approach for clustering was undertaken in \cite{Balcan:2008:DFC:1374376.1374474} within the model of similarity-based clustering.

Another approach from the information theory literature studies with online lossless compression. The relationship between compression and machine learning goes back to the Minimum Description Length criterion \cite{rissanen1978modeling}. More recent work in information theory gives online algorithms that attain optimal compression, mostly for finite alphabets \cite{DBLP:conf/isit/AcharyaDJOS13,DBLP:journals/tit/OrlitskySZ04}. For infinite alphabets, which are the main object of study for unsupervised learning of signals such as images,  there are known impossibility results \cite{DBLP:journals/tcs/JevticOS05}.  This connection to compression was recently further advanced, mostly in the context of textual data \cite{paskov2013compressive}.  

In terms of lossy compression, Rate Distortion Theory (RDT) \cite{berger1971rate,Cover:2006} is intimately related to our definitions, as a framework for finding lossy compression with minimal distortion (which would correspond to reconstruction error in our terminology). Our learnability definition can be seen of an extension of RDT to allow improper learning and generalization error bounds. Another learning framework derived from lossy compression is the information bottleneck criterion \cite{IB-tishby}, and its learning theoretic extensions \cite{Shamir2008}.  The latter framework assumes an additional feedback signal, and thus is not purely unsupervised. 

The downside of the information-theoretic approaches is that worst-case competitive compression is provably computationally hard under cryptographic assumptions. In contrast, our compression-based approach is based on learning a restriction to a specific hypothesis class, much like PAC-learning. This circumvents the impossibility results and allows for improper learning.

\section{A formal framework for unsupervised learning} 
The basis constructs in an unsupervised learning setting are: 
\begin{enumerate}
	\item
	Instance domain $\X$, such as images, text documents, etc.  Target space, or range, $\Y$.  We usually think of $\X = \reals^d, \Y= \reals^k$ with $d \gg k$.  (Alternatively, $\Y$ can be all sparse vectors in a larger space. )
	\item
	An unknown, arbitrary distribution $\D $ on domain $\X$. 
	\item
	A hypothesis class of decoding and encoding pairs, 
	$$  \H \subseteq \{ (h,g) \in \{ \X \mapsto \Y \}  \times  \{\Y \mapsto \X \} \} ,$$ 
	where $h$ is the encoding hypothesis and $g$ is the decoding hypothesis.  	\item A  loss function $\ell : \H\times \X \mapsto \reals_{\ge 0} $ that 
	measures the reconstruction error, 
	$$ \ell((g,h), x)\mper$$  For example, a natural choice is the $\ell_2$-loss $\ell((g,h), x)= \| g ( h(x)) - x \|_2^2$. The rationale here is 
		to learn structure without significantly compromising supervised learning for {\it arbitrary} future tasks. Near-perfect reconstruction is sufficient as formally proved in Appendix~\ref{appendix-reconstruction}.  Without generative assumptions, it can be seen that near-perfect reconstruction is also necessary.

	\end{enumerate}

For convenience of notation, we use $f$ as a shorthand for $(h,g)\in \cH$, a member of the hypothesis class $\H$. Denote the generalization ability of an unsupervised learning algorithm with respect to a distribution $\D$ as
$$ \loss_{\D}(f) = \E_{x \sim \D} [\ell(f , x ) ] .$$

We can now define the main object of study: unsupervised learning with respect to a given hypothesis class. The definition is parameterized by real numbers: the first is the encoding length (measured in bits) of the hypothesis class. The second is the bias, or additional error compared to the best hypothesis. Both parameters are necessary to allow improper learning.

\begin{definition}\label{def:main}
We say that instance $\D,\X$ is $(k,\gamma)$-\shmak-learnable with respect to hypothesis class $\H$ if exists an algorithm that for every $ \delta,\epsilon > 0$, after seeing $m(\epsilon,\delta) = \poly(1/\eps,\log(1/\delta),d )$ examples,  returns an encoding and decoding pair $(h,g)$ (not necessarily from $\H$) such that: 
	\begin{enumerate}
		\item
		with probability at least $1-\delta$, $ \loss_{\D} ( (h,g))  \leq \min_{ (h,g) \in \H} \loss_{\D} ((h,g))  + \epsilon + \gamma$.
		\item
		$h(x)$ has an explicit representation with length at most $k$ bits. 
\end{enumerate}

			\end{definition}

For convenience we typically encode into real numbers instead of bits. Real encoding can often (though not in the worst case) be trivially transformed to be binary with a loss of logarithmic factor.

Following PAC learning theory, we can use uniform convergence to bound the generalization error of the empirical risk minimizer (ERM). Define the empirical loss for a given sample $ S \sim \D^m$ as 
$$ \loss_{S}(f) = \frac{1}{m} \cdot \sum_{x \in S } \ell ( f ,x ) $$
Define the ERM hypothesis for a given sample $S \sim \D^m$ as 
$$ \hat{f}_{ERM} = \argmin_{\hat{f} \in \H} \loss_S(\hat{f} )  \mper $$

	For a hypothesis class $\H$, a  loss function $\ell$ and a set of $m$ samples $S\sim \D^m$, define the empirical Rademacher complexity of $\H$ with respect to $\ell$ and $S$ as, \footnote{Technically, this is the Rademacher complexity of the class of functions $\ell \circ \cH$. However, since $\ell$ is usually fixed for certain problem, we emphasize in the definition more the dependency on $\cH$.}
	$$ \mathcal{R}_{S,\ell}(\H) = \E _{\sigma \sim \{ \pm 1\}^m } \left[ \sup_{f \in \H} \frac{1}{m} \sum_{x \in S} \sigma_i \ell( f ,x ) \right]   $$ 
	Let the Rademacher complexity of $\H$ with respect to distribution $\D$ and loss $\ell$  as	$ \mathcal{R}_m(\H) = \E_{S \sim \D^m} [ \mathcal{R} _{S,\ell}(\H) ] $. When it's clear from the context, we will omit the subscript $\ell$. 

We can now state and apply standard generalization error results. The proof of following theorem is almost identical to \cite[Theorem 3.1]{mohri2012foundations}. For completeness we provide a proof in Appendix~\ref{subsec:generalization_proof}. 
\begin{theorem}\label{thm:rademacher-generalization}
	For any $\delta > 0$, with probability $1 - \delta$, the generalization error of the ERM hypothesis is bounded by: 
	$$ \loss_{\D} (  \hat{f}_{ERM} ) \leq \min_{f \in \H} \loss_{\D} (f)  + 6 \mathcal{R}_m(\H) + \sqrt{ \frac{4 \log \frac{1}{\delta}}{2 m }} $$
\end{theorem}

An immediate corollary of the theorem is that as long as the Rademacher complexity of a hypothesis class approaches zero as the number of examples goes to infinity, it can be \shmak~learned by an \textit{inefficient} algorithm that optimizes over the hypothesis class by enumeration and outputs an best hypothesis with encoding length $k$ and bias $\gamma = 0$. Not surprisingly such optimization is often intractable and hences the main challenge is to design efficient algorithms. As we will see in later sections, we often need to trade the encoding length and bias slightly for computational efficiency.

\paragraph{Notation}

For every vector $z \in \R^{d_1}\otimes \R^{d_2}$, we can view it as a matrix of dimension $d_1\times d_2$, which is denoted as $\mat(z)$. Therefore in this notation, $\mat(u\otimes v) = uv^{\top}$. 

Let $v_{\max}(\cdot) : (\R^{d})^{\otimes 2}\rightarrow \R^d$ be the function that compute the top right-singular vector of some vector in $(\R^{d})^{\otimes 2}$ viewed as a matrix. That is, for $z\in (\R^{d})^{\otimes 2}$, then $v_{\max}(z)$ denotes the top right-singular vector of $\mat(z)$. We also overload  the notation $v_{\max}$ for generalized eigenvectors of higher order tensors. For $T\in (\R^d)^{\otimes \ell}$, let $v_{\max}(T) = \argmax_{\|x\|\le 1} T(x,x,\dots,x)$ where $T(\cdot)$ denotes the multi-linear form defined by tensor $T$.

We use $\|A\|_{\ell_p\to \ell_q}$ to denote the induced operator norm of $A$ from $\ell_p$ space to $\ell_q$ space. For simplicity, we also define
$|A|_1  = \|A\|_{\ell_{\infty}\to \ell_{1}} = \sum_{ij} |A_{ij}|$, $|A|_{\infty} = \|A\|_{\ell_1\to\ell_{\infty}} = \max_{ij}|A_{ij}|$. We note that $\|A\|_{\ell_1\to \ell_1}$ is the max column $\ell_1$ norm of $A$, and $\|A\|_{\ell_1\to \ell_2}$ is the largest column $\ell_2$ norm of $A$.

\section{Spectral autoencoders: unsupervised learning of algebraic manifolds}
\renewcommand{\ss}{s}
\subsection{Algebraic manifolds} 

The goal of the spectral autoencoder hypothesis class we define henceforth is to learn the representation of data that lies on a low-dimensional algebraic variety/manifolds. The linear variety, or linear manifold, defined by the roots of linear equations, is simply a linear subspace. If the data resides in a linear subspace, or close enough to it, then PCA is effective at learning its succinct representation. 

One extension of the linear manifolds is the set of roots of low-degree polynomial equations. Formally, let $k,s$ be integers and let $c_1,\dots, c_{d^s-k}\in \R^{d^s}$ be a set of vectors in $d^s$ dimension, and consider the algebraic variety 
\begin{align}
\mathcal{M} = \Set{x\in \R^d: \forall i\in [d^s-k], \inner{c_i, x^{\otimes s} }=0}\nonumber\mper
\end{align}
Observe that here each constraint $\inner{c_i, x^{\otimes s}}$ is a degree-$s$ polynomial over variables $x$, and when $s= 1$ the variety $\mathcal{M}$ becomes a liner subspace. Let $a_1,\dots, a_k\in \R^{d^s}$ be a basis of the subspaces orthogonal to all of $c_1,\dots, c_{d^s-k}$, and let $A\in \R^{k\times d^s}$ contains $a_i$ as rows. Then we have that given $x\in \mathcal{M}$, the encoding 
\begin{align}
y = Ax^{\otimes s} \nonumber
\end{align}
pins down all the unknown information regarding $x$. In fact, for any $x\in \mathcal{M}$, we have $A^{\top}Ax^{\otimes s} = x^{\otimes s}$ and therefore $x$ is decodable from $y$. The argument can also be extended to the situation when the data point is close to $\mathcal{M}$ (according to a metric, as we discuss later). The goal of the rest of the subsections is to learn the encoding matrix $A$ given data points residing close to $\mathcal{M}$.

\subsection{Warm up: PCA and kernel PCA}
In this section we illustrate our framework for agnostic unsupervised learning by showing how PCA and kernel PCA can be efficiently learned within our model.  The results of this sub-section are not new, and given only for illustrative purposes. 

The class of hypothesis corresponding to PCA operates on domain $\X = \reals^d$ and range $\Y = \reals^k$ for some $k < d$ via linear operators. In kernel PCA, the encoding linear operator applies to the $\ss$-th tensor power $x\tp{\ss}$ of the data. That is, the encoding and decoding are parameterized by a linear operator $A \in \reals ^{k \times d^{\ss}}$,  $$\H_{k,\ss}^{\pca} = \Set{(h_A, g_A) : h_A( x) = A x\tp{\ss}, \ , g_A ( y) = A^{\dagger}y}\mcom$$
where $A^{\dagger}$ denotes the pseudo-inverse of $A$. 
The natural loss function here is the Euclidean norm,  $$\ell((g,h),x)  = \|x\tp{\ss} - g(h(x))\|^2 = \|(I-A^{\dagger}A)x\tp{\ss}\|^2\mper$$ 

\begin{theorem}\label{thm:pca}
For a fixed constant $\ss\ge 1$, the class $\H_{k,\ss}^{\pca} $ is efficiently  \shmak-learnable with encoding length $k$ and bias $\gamma =0$. 
\end{theorem}

The proof of the Theorem follows from two simple components: a) finding the ERM among $\cH_{k,\ss}^{\pca}$ can be efficiently solved by taking SVD of covariance matrix of the (lifted) data points. b) The Rademacher complexity of the hypothesis class is bounded by $O(d^{\ss}/m)$  for $m$ examples. Thus by Theorem~\ref{thm:rademacher-generalization} the minimizer of ERM generalizes. The full proof is deferred to Appendix~\ref{sec:proof_pca}.

\subsection{Spectral Autoencoders}

In this section we give a much broader set of hypothesis, encompassing PCA and kernel-PCA, and show how to learn them efficiently. 
Throughout this section we assume that  the data is normalized to Euclidean norm 1, and consider the following class of hypothesis which naturally generalizes PCA:

\begin{definition}[Spectral autoencoder]
	We define the class $\cH_{k,\ss}^{\spd}$ as the following set of all hypothesis $(g,h)$,		\begin{equation}
	\cH^{\spd}_{k} = \left\{(h,g):  \begin{array}{ll} h(x ) & = Ax\tp{\ss},  A\in \R^{k\times d^{\ss}}\\
	g(y) & = v_{\max}(By), B\in \R^{d^{\ss}\times k}\end{array} \right\}\mper
	\end{equation}
								\end{definition}

We note that this notion is more general than kernel PCA: suppose some $(g,h)\in \cH_{k,\ss}^{\pca}$ has reconstruction error $\eps$,  namely, $A^{\dagger}Ax\tp{\ss}$ is $\eps$-close to $x\tp{\ss}$ in Euclidean norm. Then by eigenvector perturbation theorem, we have that $v_{\max}(A^{\dagger}Ax\tp{\ss})$ also reconstructs $x$ with $O(\eps)$ error, and therefore there exists a PSCA hypothesis with $O(\eps)$ error as well . Vice versa, it's quite possible that for every $A$, the reconstruction $A^{\dagger}Ax\tp{\ss}$ is far away from $x\tp{\ss}$ so that kernel PCA doesn't apply, but with spectral decoding we can still reconstruct $x$ from $v_{\max}(A^{\dagger}Ax\tp{\ss})$ since the top eigenvector of $A^{\dagger}Ax\tp{\ss}$ is close $x$.

Here the key matter that distinguishes us from kernel PCA is in what metric $x$ needs to be close to the manifold so that it can be reconstructed. Using PCA, the requirement is that $x$ is in Euclidean distance close to $\cM$ (which is a subspace), and using kernel PCA $x\tp{2}$ needs to be in Euclidean distance close to the null space of $c_i$'s. However, Euclidean distances in the original space and lifted space typically are meaningless for high-dimensional data since any two data points are far away with each other in Euclidean distance. The advantage of using spectral autoencoders is that in the lifted space the geometry is measured by spectral norm distance that is much smaller than Euclidean distance (with a potential gap of $d^{1/2}$). The key here is that though the dimension of lifted space is $d^{2}$, the objects of our interests is the set of rank-1 tensors of the form $x\tp{2}$. Therefore, spectral norm distance is a much more effective measure of closeness since it exploits the underlying structure of the lifted data points. 

We note that spectral autoencoders relate to vanishing component analysis~\cite{DBLP:conf/icml/LivniLSNSG13}. When the data is close to an algebraic manifold, spectral autoencoders aim to find the (small number of) essential non-vanishing components in a noise robust manner.

\subsection{Learnability of polynomial spectral decoding }

For simplicity we focus on the case when $\ss = 2$. Ideally we would like to learn the best encoding-decoding scheme for any data distribution $\D$. Though there are technical difficulties to achieve such a general result. 
A natural attempt would be to optimize the loss function $f(A,B)=\|g(h(x)) - x\|^2 = \|x-v_{\max}(BAx\tp{2})\|^2$. Not surprisingly, function $f$ is not a convex function with respect to $A,B$, and in fact it could be even  non-continuous (if not ill-defined)!

Here we make a further realizability assumption that the data distribution $\D$  admits a reasonable encoding and decoding pair with reasonable reconstruction error. \begin{definition}\label{def:regularly_decodable}
	We say a data distribution $\D$ is $(k,\epsilon)$-\textit{regularly} spectral decodable if there exist $A\in \R^{k\times d^2}$ and $B\in \R^{d^2\times k}$ with $\|BA\|_{\op}\le \tau$ such that for $x\sim \D$, with probability 1, the encoding $y = Ax\tp{2}$ satisfies that 
	\begin{equation}
	\mat(By) = \mat(BAx\tp{2}) = xx^{\top} + E \label{eqn:2}\mcom
	\end{equation}
	where $\|E\|_{\op}\le \epsilon$. Here $\tau\ge 1$ is treated as a fixed constant globally.  
\end{definition}

To interpret the definition, we observe that if data distribution $\cD$ is $(k,\epsilon)$-regularly spectrally decodable, then by equation~\eqref{eqn:2} and Wedin's theorem (see e.g. \cite{vu-wedin} )  on the robustness of eigenvector to perturbation, $\mat(By)$ has top eigenvector\footnote{Or right singular vector when $\mat(By)$ is not symmetric} that is $O(\epsilon)$-close to $x$ itself. Therefore, definition~\ref{def:regularly_decodable} is a sufficient condition for the spectral decoding algorithm $v_{\max}(By)$ to return $x$ approximately, though it might be not necessary.  Moreover, this condition partially addresses the non-continuity issue of using objective $f(A,B)= \|x-v_{\max}(BAx\tp{2})\|^2$, while $f(A, B)$ remains (highly) non-convex. We resolve this issue by using a convex surrogate.

Our main result concerning the learnability of the aforementioned hypothesis class is: 

\begin{theorem}\label{thm:main-sdp}
		The hypothesis class $\H^{\spd}_{k,2}$ is \shmak - learnable with encoding length $O(\tau^4 k^4/\delta^4)$ and bias $\delta$ with respect to $(k,\epsilon)$-regular distributions in polynomial time. 
\end{theorem}

Our approach towards finding an encoding and decoding matrice $A,B$ is to optimize the objective, 
\begin{align}
\textup{minimize}~f(R) & = \Exp\left[\left\|Rx\tp{2} - x\tp{2}\right\|_{\tsop}\right] \label{eqn:objective_relaxation}\\
\textup{s.t.}~& \norm{R}_{S_1} \le \tau k\nonumber
\end{align}
where $\|\cdot\|_{S_1}$ denotes the Schatten 1-norm\footnote{Also known as nuclear norm or trace norm}. 

Suppose $\cD$ is $(k,\epsilon)$-regularly decodable, and suppose $h_A$ and $g_B$ are the corresponding encoding and decoding function. Then we see that $R = AB$ will satisfies that $R$ has rank at most $k$ and $f(R)\le \epsilon$.  On the other hand, suppose one obtains some $R$ of rank $k'$ such that $f(R)\le \delta$, then we can produce $h_A$ and $g_B$ with $O(\delta)$ reconstruction simply by choosing $A\in \R^{k'\times d^2}B$ and $B\in \R^{d^2 \times k'}$ such that $R = AB$. 

We use (non-smooth) Frank-Wolfe to solve objective~\eqref{eqn:objective_relaxation}, which in particular returns a low-rank solution. We defer the proof of Theorem~\ref{thm:main-sdp} to the Appendix~\ref{sec:main_sdp_proof}. 

With a slightly stronger assumptions on the data distribution $\D$, we can reduce the length of the code to $O(k^2/\eps^2)$ from $O(k^4/\eps^4)$. See details in Appendix~\ref{sec:shorter_code_appendix}.

\section{A family of optimization encodings and efficient dictionary learning}\label{sec:dict}

In this section we give efficient algorithms for learning a family of unsupervised learning algorithms commonly known as "dictionary learning". In contrast to previous approaches, we do not construct an actual "dictionary", but rather improperly learn a comparable encoding via convex relaxations. 

We consider a different family of codes which is motivated by matrix-based unsupervised learning models such as topic-models, dictionary learning and PCA. This family is described by a matrix $A \in \reals^{d \times r}$ which has low complexity according to a certain norm $\|\cdot\|_{\alpha}$, that is, $ \|A\|_{\alpha} \leq c_{\alpha}$. 
We can parametrize a family of hypothesis $\H$ according to these matrices, and define an encoding-decoding pair according to 
$$ h_A(x) = \argmin_{\|y \|_{\beta} \leq k } \frac{1}{d}\gNorm{ x- A y} \ , \   g_A(y) = Ay$$ 
We choose $\ell_1$ norm to measure the error mostly for convenience, though it can be quite flexible. The different norms $\|\cdot\|_{\alpha} , \|\cdot\|_{\beta}$ over $A$ and $y$ give rise to different learning models that have been considered before. For example, if these are Euclidean norms, then we get PCA. If $\|\cdot \|_{\alpha}$ is the max column $\ell_{2}$ or $\ell_{\infty}$ norm and $\|\cdot\|_{b}$ is the $\ell_0$ norm, then this corresponds to dictionary learning (more details in the next section).

The optimal hypothesis in terms of reconstruction error is given by $$ A ^{\star} = \argmin_{\|A\|_{\alpha}\le c_{\alpha}}\Exp_{x\sim \mathcal{D}}\left[\frac{1}{d}\gNorm{x - g_A(h_A(x))}\right] = \argmin_{\|A\|_{\alpha}\le c_{\alpha}}\Exp_{x\sim \mathcal{D}}\left[\min_{y\in \R^r: \| y\|_{\beta} \leq k } \frac{1}{d}\gNorm{ x- A y} \right]\mper$$

The loss function can be generalized to other norms, e.g., squared $\ell_2$ loss, without any essential change in the analysis.  Notice that this optimization objective derived from reconstruction error is identically the one used in the literature of dictionary learning. This can be seen as another justification for the definition of unsupervised learning as minimizing reconstruction error subject to compression constraints. 

The optimization problem above is notoriously hard computationally, and significant algorithmic and heuristic literature attempted to give efficient algorithms under various distributional assumptions(see~\cite{arora2013new,DBLP:conf/colt/AroraGMM15, Aharon05k-svd-design} and the references therein).  Our approach below circumvents this computational hardness by convex relaxations that result in learning a different creature, albeit with comparable compression and reconstruction objective.

\subsection{Improper dictionary learning: overview}

We assume the max column $\ell_{\infty}$ norm of $A$ is at most $1$ and the $\ell_1$ norm of $y$ is assumed to be at most $k$. This is a more general setting than the  random dictionaries (up to a re-scaling) that previous works~\cite{arora2013new,DBLP:conf/colt/AroraGMM15} studied. \footnote{The assumption can be relaxed to that $A$ has $\ell_{\infty}$ norm at most $k$ and $\ell_2$-norm at most $\sqrt{d}$ straightforwardly. }In this case, the magnitude of each entry of $x$ is on the order of $\sqrt{k}$ if $y$ has $k$ random $\pm 1$ entries. We think of our target error per entry as much smaller than 1\footnote{We are conservative in the scaling of the error here. Error much smaller than $\sqrt{k}$ is already meaningful. }. 
We consider $\H^k_{\textup{dict}}$ that are parametrized by the dictionary matrix $A = \R^{d\times r}$, \begin{align}
& \H^{\textup{dict}}_k = \left\{(h_A,g_A): A\in \R^{d\times r}, \|A\|_{\ell_1 \to \ell_{\infty}}\le 1\right\}\mcom\nonumber\\
& \textup{ where }h_A(x) = \argmin_{\|y \|_{1} \leq k } \gNorm{x- A y} \ , \   g_A(y) = Ay\nonumber
\end{align}

Here we allow $r$ to be larger than $d$, the case that is often called over-complete dictionary. 
The choice of the loss can be replaced by $\ell_2$ loss (or other Lipschitz loss) without any additional efforts, though for simplicity we stick to $\ell_1$ loss. Define $A^{\star}$ to be the the best dictionary under the model and $\epsilon^{\star}$ to be the optimal error, 
\begin{eqnarray} \label{eqn:opterror}
& A ^{\star} = \argmin_{\|A\|_{\ell_1 \to \ell_\infty}\le 1}\Exp_{x\sim \mathcal{D}}\left[\min_{y\in \R^r: \| y\|_{1} \leq k } \gNorm{x- A y} \right]  \\
& \eps^{\star} = \Exp_{x\sim \mathcal{D}}\left[ \frac{1}{d}\cdot \gNorm{x-g_{A^{\star}}(h_{A^{\star}}(x))}\right]\mper  \notag
\end{eqnarray}

\begin{algorithm}\caption{group encoding/decoding for improper dictionary learning} \label{alg:improper_dict}
	{\bf Inputs: } $N$ data points $X \in \R^{d\times N}\sim \cD^N$. Convex set $Q$. Sampling probability $\rho$. 
	\begin{enumerate}
																		
		\item {\bf Group encoding: } 
		Compute 
				\begin{equation}
		Z = 	\argmin_{C\in Q} \gNorm{X-C}\mcom \label{eqn:eqn11} 		\end{equation}
		and let 
		\begin{equation}
		Y = h(X) = \Po(Z)\nonumber \mcom
		\end{equation}
		where $\Po(B)$ is a random sampling of $B$ where each entry is picked with probability $\rho$. 
		\item {\bf Group decoding: }
		Compute 
		\begin{equation}
		g(Y) = \argmin_{C\in Q} \gNorm{\Po(C) -Y} 
				\mper \label{eqn:eqn25}
		\end{equation}
											\end{enumerate}
\end{algorithm}

\begin{theorem}\label{thm:single}
	For any $\delta > 0, p\ge 1$, the hypothesis class $\H^{\textup{dict}}_k$ is \shmak-learnable (by Algorithm~\ref{alg:improper_dict_single}) with encoding length $\tilde{O}(k^2r^{1/p}/\delta^2)$,  bias $\delta + O(\epsilon^{\star})$ and sample complexity $d^{O(p)}$ in time $n^{O(p^2)}$   
\end{theorem}

We note that here $r$ can be potentially much larger than $d$ since by choosing a large constant $p$ the overhead caused by $r$ can be negligible. Since the average size of the entries is $\sqrt{k}$, therefore we can get the bias $\delta$ smaller than average size of the entries with code length roughly $\approx k$. 

The proof of Theorem~\ref{thm:single} is deferred to Section~\ref{subsec:single}. To demonstrate the key intuition and technique behind it, in the rest of the section we consider a simpler algorithm that achieves a \textit{weaker} goal: Algorithm~\ref{alg:improper_dict} encodes \textit{multiple} examples into some codes with the matching average encoding length $\tilde{O}(k^2r^{1/p}/\delta^2)$, and these examples can be decoded from the codes together with reconstruction error $\epsilon^{\star} + \delta$. Next, we outline the analysis of Algorithm~\ref{alg:improper_dict}, and we will show later that one can reduce the problem of encoding a single examples to the problem of encoding multiple examples together.

Here we overload the notation $g_{A^{\star}}(h_{A^{\star}}(\cdot))$ so that $g_{A^{\star}}(h_{A^{\star}}(X))$ denotes the collection of all the $g_{A^{\star}}(h_{A^{\star}}(x_j))$ where $x_j$ is the $j$-th column of $X$. 
Algorithm~\ref{alg:improper_dict} assumes that there exists 
a convex set $Q\subset \R^{d\times N}$ such that \begin{equation}
\Set{g_{A^{\star}}(h_{A^{\star}}(X)): X\in \R^{d\times N} }\subset \Set{AY : \|A\|_{\ell_1 \to \ell_{\infty}} \le 1, \|Y\|_{\ell_1 \to \ell_{1}}\le  k}\subset Q\mper \label{eqn:feasibility}
\end{equation}

That is, $Q$ is a convex relaxation of the group of reconstructions allowed in the class $\cH^{\textup{dict}}$.  
Algorithm~\ref{alg:improper_dict} first uses convex programming to denoise the data $X$ into a clean version $Z$, which belongs to the set $Q$. If the set $Q$ has low complexity, then simple random sampling of $Z\in Q$ serves as a good encoding.

The following Lemma shows that if $Q$ has low complexity in terms of sampling Rademacher width, then Algorithm~\ref{alg:improper_dict} will give a good group encoding and decoding scheme. 

\begin{lemma}\label{thm:general}
	Suppose convex $Q\subset \R^{d\times N}$ satisfies condition~\eqref{eqn:feasibility}. Then, Algorithm~\ref{alg:improper_dict} gives a group encoding and decoding pair such that with probability $1-\delta$, the average reconstruction error is bounded by $\epsilon^{\star} + O(\sqrt{\SRC_m(Q}) + O(\sqrt{\log(1/\delta)/m})$ where $m =\rho Nd$ and $\SRC_m(\cdot)$ is the sampling Rademacher width (defined in subsection~\ref{subsec:rademacher}), and the average encoding length is $\tilde{O}(\rho d)$. 
		\end{lemma}

The proofs here are technically standard: Lemma~\ref{thm:general} simply follows from Lemma~\ref{lem:denoising} and Lemma~\ref{lem:rade} in Section~\ref{sec:dict_appendix}. Lemma~\ref{lem:denoising} shows that the difference between $Z$ and $X$ is comparable to $\eps^{\star}$, which is a direct consequence of the optimization over a large set $Q$ that contains optimal reconstruction. Lemma~\ref{lem:rade} shows that the sampling procedure doesn't lose too much information given a denoised version of the data is already observed, and therefore, one can reconstruct $Z$ from $Y$. 

The novelty here is to use these two steps together to denoise and achieve a short encoding. The typical bottleneck of applying convex relaxation on matrix factorization based problem (or any other problem) is the difficulty of rounding. Here instead of pursuing a rounding algorithm that output the factor $A$ and $Y$, we look for a convex relaxation that preserves the intrinsic complexity of the set which enables the trivial sampling encoding. It turns out that controlling the width/complexity of the convex relaxation boils down to proving concentration inequalities with sum-of-squares (SoS) proofs, which is conceptually easier than rounding. 

Therefore, the remaining challenge is to design convex set $Q$ that simultaneously has the following properties
\begin{enumerate}
	\item[(a)] is a convex relaxation in the sense of satisfying  condition~\eqref{eqn:feasibility}
	\item[(b)] admits an efficient optimization algorithm
	\item[(c)] has low complexity (that is, sampling rademacher width $\tilde{O}(N\poly(k))$)
\end{enumerate} 
Most of the proofs need to be deferred to Section~\ref{sec:dict_appendix}. We give a brief overview: In subsection~\ref{subsec:norms}  we will design a convex set $Q$ which satisfies condition~\eqref{eqn:feasibility} but not efficiently solvable, and in subsection~\ref{subsec:unrelaxed_rade} we verify that the sampling Rademacher width is $O(Nk\log d)$. In subsection~\ref{subsec:relaxed_rade}, we prove that a sum-of-squares relaxation would give a set  $Q^{\sos}_{p}$ which satisfies (a), (b) and (c) approximately. Concretely, we have the following theorem. 

\begin{theorem}\label{thm:sos}
	For every $p\ge 4$, let $N = d^{c_0p}$ with a sufficiently large absolute constant $c_0$. Then,  there exists a convex set $Q^{\sos}_{p}\subset \R^{d\times N}$ (which is defined in subsection~\ref{subsec:relaxation}) such that (a) it satisfies condition~\ref{eqn:feasibility}; (b)  The optimization~\eqref{eqn:eqn11} and~\eqref{eqn:eqn25} are solvable by semidefinite programming with run-time $n^{O(p^2)}$; (c) the sampling Rademacher width of $Q^{\sos}_{p}$ is bounded by $\sqrt{\SRC_m(Q)}\le \tilde{O}(k^2r^{2/p}N/m)$.
	\end{theorem}
We note that these three properties (with Lemma~\ref{thm:general}) imply that Algorithm~\ref{alg:improper_dict} with $Q = Q^{\sos}_{p}$ and $\rho = O(k^2r^{2/p}d^{-1}/\delta^2 \cdot \log d)$ gives a group encoding-decoding pair with average encoding length $O(k^2r^{2/p}/\delta^2 \cdot \log d)$ and bias $\delta.$

{\bf Proof Overview of Theorem~\ref{thm:sos}: } At a very high level, the proof exploits the duality between sum-of-squares relaxation and sum-of-squares proof system. Suppose $w_1,\dots, w_d$ are variables, then in SoS relaxation an auxiliary variable $W_S$ is introduced for every subset $S\subset [d]$ of size at most $s$, and valid linear constraints and psd constraint for $W_S$'s are enforced. By convex duality, intuitively we have that if a polynomial $q(x) = \sum_{|S|\le s} \alpha_Sx_S$ can be written as a sum of squares of polynomial $q(x) = \sum_{j} r_j(x)^2$, then the corresponding linear form over $X_S$, $\sum_{|S|\le s} \alpha_SX_S$ is also nonnegative. Therefore, to certify certain property of  a linear form $\sum_{|S|\le s} \alpha_SX_S$ over $X_S$, it is sufficient (and also necessary by duality) that the corresponding polynomial admit a sum-of-squares proof. 

Here using the idea above, we first prove the Rademacher width of the convex hull of reconstructions $Q_0=\conv\Set{Z = AY: \|A\|_{\ell_1\to \ell_{\infty}}\le 1, \|Y\|_{\ell_1\to \ell_1}\le k}$ using a SoS proof. Then the same proof automatically applies to for the Rademacher width of the convex relaxation (which is essentially a set of statements about linear combinations of the auxiliary variables).  We lose a factor of $r^{2/p}$ because SoS proof is not strong enough for us to establish the optimal Rademacher width of $Q_0$.

\section{Analysis of Improper Dictionary Learning}
\label{sec:dict_appendix}

In this section we give the full proof of the Theorems and Lemmas in Section~\ref{sec:dict}.  We start by stating general results on denoising, Rademacher width and factorizable norms, and proceed to give specialized bounds for our setting in section \ref{subsec:unrelaxed_rade}.

\subsection{Guarantees of denoising}\label{subsec:denoising}

In this subsection, we give the guarantees of the error caused by the denoising step. Recall that $\epsilon^\star$ is the optimal reconstruction error achievable by the optimal (proper) dictionary (equation \eqref{eqn:opterror}). 
\begin{lemma}\label{lem:denoising} 
	Let $Z$ be defined in equation~\eqref{eqn:eqn11}. Then we have that 
	\begin{equation}
	\frac{1}{Nd}\Exp_{X\sim \cD^N}\left[\gNorm{Z-X}\right]\le \epsilon^{\star}
	\end{equation}
\end{lemma}
\begin{proof}
	Let $Y^{\star} = A^{\star} h_{A^{\star}}(X)$ where $h_{A^{\star}} (X)$ denote the collection of encoding of $X$ using $h_{A^{\star}}$. Since $Y^{\star}\in \Set{AY : \|A\||_{\ell_1\to \ell_{\infty}} \le 1, \|Y\|_{\ell_1 \to \ell_{1}}\le  k}\subset Q $, we have that $Y^{\star}$ is a feasible solution of optimization~\eqref{eqn:eqn11}.  Therefore, we have that $\frac{1}{Nd}\Exp\left[\gNorm{Z-X}\right]\le \frac{1}{Nd}\Exp\left[\gNorm{X-Y^{\star}}\right] = \epsilon^{\star}$, where the equality is by the definition of $\epsilon^{\star}$. 
\end{proof}

\subsection{Sampling Rademacher Width of a Set}\label{subsec:rademacher}

As long as the intrinsic complexity of the set $Q$ is small then we can compress by random sampling. The idea of viewing reconstruction error the test error of a supervised learning problem started with the work of Srebro and Shraibman~\cite{DBLP:conf/colt/SrebroS05}, and has been used for other completion problems, e.g., ~\cite{DBLP:journals/corr/BarakM15}. We use the terminology ``Rademacher width'' instead of ``Rademacher complexity'' to emphasize that the notation defined below is a property of a set of vectors (instead of that of a hypothesis class).

For any set $W\subset \R^D$, and an integer $m$,  we define its sampling Rademacher width (SRW) as, 
\begin{equation}
\SRC_m(W) = \Exp_{\sigma,\Omega}\left[\frac{1}{m}\sup_{x \in W}\inner{x,\sigma}_{\Omega}\right]\mcom
\end{equation}

where $\Omega$ is random subset of $[D]$ of size $m$, $\inner{a,b}_{\Omega}$ is defined as $\sum_{i\in \Omega}a_ib_i$ and $\sigma\sim\{\pm1\}^D$. 

\begin{lemma}(\cite[Theorem 2.4]{DBLP:journals/corr/BarakM15})\label{lem:src}\label{lem:rade}	With probability at least $1-\delta$ over the choice of $\Omega$, for any $\tilde{z}\in \R^D$ 	\begin{equation}
		\frac{1}{D}|\tilde{z} - z|_1\le \frac{1}{m}|\Po(\tilde{z})-\Po(z)|_1 + 2\SRC_m(W) + M\sqrt{\frac{\log(1/\delta)}{m}}\mper\nonumber
	\end{equation}
	where $M = \sup_{\tilde{z}\in W,i\in [D]}|z_i-\tilde{z}_i|$. 
	\end{lemma}

\subsection{Factorable norms}\label{subsec:norms}

In this subsection, we define in general the factorable norms, from which we obtain a convex set $Q$ which satisfies condition~\eqref{eqn:feasibility} (see Lemma~\ref{cor:convex_set}). 

For any two norms $\|\cdot\|_{\alpha}, \|\cdot\|_{\beta}$ that are defined on matrices of any dimension,  we can define the following quantity

\begin{equation}
\Gamma_{\alpha,\beta}(Z) = \inf_{Z = AB} \|A\|_{\alpha}\norm{B}_{\beta} \label{eqn:def-factorial-norm}
\end{equation}

For any $p,q,s,t\ge 1$,  we use $\Gamma_{p,q,s,t}(\cdot)$ to denote the function $\Gamma_{\ell_p\to \ell_q, \ell_s\to\ell_t}(\cdot)$. When $p = t$, $\Gamma_{p,q,s,p}(Z)$ is the factorable norm~\cite[Chapter 13]{tomczak1989banach} of matrix $Z$. In the case when $p = t = 2$, $q = \infty, s=1$, we have that $\Gamma_{2,\infty, 1,2}(\cdot)$ is the $\gamma_2$-norm~\cite{DBLP:journals/combinatorica/LinialMSS07} or max norm~\cite{DBLP:conf/colt/SrebroS05}, which has been used for matrix completion. 

The following Lemma is the main purpose of this section which shows a construction of a convex set $Q$ that satisfies condition~\eqref{eqn:feasibility}. 
\begin{lemma}\label{cor:convex_set}
	For any $q,t\ge 1$ we have that $\Gamma_{1,q,1,t}(\cdot)$ is a norm. As a consequence, letting
	$Q_{1,\infty,1,1}= \{C\in \R^{N\times d}:\Gamma_{1,\infty,1,1}(C)\le \sqrt{d}k\}$, 	we have that $Q_{1,\infty,1,1}$  is
	a convex set and it 	satisfies condition~\eqref{eqn:feasibility}. 
\end{lemma}

Towards proving Lemma~\ref{cor:convex_set}, we prove a stronger result that if $p=s=1$, then $\Gamma_{p,q,s,t}$ is also a norm. This result is parallel to ~\cite[Chapter 13]{tomczak1989banach} where the case of
 $p=t$ is considered. 

\begin{theorem}\label{thm:norm}
	Suppose that $\|\cdot\|_{\alpha}$ and $\|\cdot\|_{\beta}$ are norms defined on matrices of any dimension such that 
	\begin{enumerate}
	\item $\|[A,B]\|_{\alpha}\le \max\{\|A\|_{\alpha},\|B\|_{\alpha}\}$
	\item  $\|\cdot\|_{\beta}$ is invariant with respect to appending a row with zero entries.
	\end{enumerate}
	Then, $\Gamma_{\alpha,\beta}(\cdot)$ is a norm. 	\end{theorem}
\begin{proof}
	
	{\bf Non-negativity:} 	We have that $\Gamma_{\alpha,\beta}(Z)\ge 0$ by definition as a product of two norms. Further, $\Gamma_{\alpha,\beta}(Z)=0$ if and only if $\|A\|_\alpha=0$ or $\|B\|_\beta=0$, which is true if and only if $A$ or $B$ are zero, which means that $Z$ is zero.

	{\bf Absolute scalability:}  	For any positive $t$, we have that if $Z = AB$ and $\Gamma_{\alpha,\beta}(Z) \le \norm{A}_{\alpha}\norm{B}_{\beta}$, then  $tZ = (tA)\cdot B$ and $\Gamma_{\alpha,\beta}(tZ) \le t\norm{A}_{\alpha}\norm{B}_{\beta}$. Therefore by definition of $\Gamma_{\alpha,\beta}(Z)$, we get $\Gamma_{\alpha,\beta}(tZ) \le t\Gamma_{\alpha,\beta}(Z)$. 
	
	If we replace $Z$ by $tZ$ and $t$ by $1/t$ we obtain the other direction, namely  $\Gamma_{\alpha,\beta}(Z) \le 1/t\cdot\Gamma_{\alpha,\beta}(tZ)$. Therefore, $\Gamma_{\alpha,\beta}(t Z) = t \Gamma_{\alpha,\beta}(Z)$ for any $t \ge 0$.

	{\bf Triangle inequality:} 		We next show that $\Gamma_{\alpha,\beta}(Z)$ satisfies triangle inequality, from which the result follows. Let $W$ and $Z$ be two matrices of the same dimension. Suppose $A,C$ satisfy that $Z= AC$ and $\Gamma_{\alpha,\beta}(Z) =  \|A\|_{\alpha}\|C\|_{\beta}$. Similarly, suppose $W = BD$, and $\Gamma_{\alpha,\beta}(W) =  \|B\|_{\alpha}\|D\|_{\beta}$. Therefore, we have that $W+Z = [tA,B]\begin{bmatrix}
	t^{-1}C\\D
	\end{bmatrix}$, and that for any $t > 0$, 
	\begin{align}
		\Gamma_{\alpha,\beta}(W+Z) & \le \Norm{[tA,B]}_{\alpha}\Norm{\begin{bmatrix}
			t^{-1}C\\D
			\end{bmatrix}}_{\beta}\tag{by defintion of 	$\Gamma_{\alpha,\beta}$}	\\
		& \le  \Norm{[tA,B]}_{\alpha}\Paren{\Norm{\begin{bmatrix}
			0\\D
			\end{bmatrix}}_{\beta} + t^{-1}\Norm{\begin{bmatrix}
			C\\0
			\end{bmatrix}}_{\beta} } \tag{by triangle inquality}\\
		& \le \max\left\{t\Norm{A}_{\alpha}, \Norm{B}_{\alpha}\right\}\Paren{t^{-1}\Norm{C}_{\beta} + \Norm{
			D}_{\beta}} 	\tag{by assumptions on $\|\cdot\|_{\alpha}$ and $\|\cdot\|_{\beta}$}
	\end{align}
	Pick $t  = \frac{\Norm{B}_{\alpha}}{\Norm{A}_{\alpha}}$, we obtain that, 
	\begin{align}
	\Gamma_{\alpha,\beta}(W+Z) & \le \Norm{A}_{\alpha}\Norm{C}_{\beta}+ \Norm{B}_{\alpha}\Norm{D}_{\beta} \nonumber\\
	& = \Gamma_{\alpha,\beta}(Z) + \Gamma_{\alpha,\beta}(W)\nonumber \mper
		\end{align}
\end{proof}

Note that if $\|\cdot \|_{\alpha}$ is a $\ell_1\to \ell_q$ norm, then it's also the max column-wise $\ell_q$ norm, and therefore it satisfies the condition a) in Theorem~\ref{thm:norm}. Moreover, for similar reason, $\|\cdot\|_{\beta} = \|\cdot\|_{\ell_1\to \ell_t}$ satisfies the condition b) in Theorem~\ref{thm:norm}. Hence, Lemma~\ref{cor:convex_set} is a direct consequence of Theorem~\ref{thm:norm}.  Lemma~\ref{cor:convex_set} gives a  convex set $Q$ that can be potentially used in Algorithm~\ref{alg:improper_dict}.

\subsection{Sampling Rademacher width of level set of $\Gamma_{1,\infty, 1,1}$}\label{subsec:unrelaxed_rade}

Here we give a Rademacher width bound for the specific set we're interested in, namely the level sets of $\Gamma_{1,\infty, 1,1}$, formally,
$$Q_{1,\infty,1,1}= \{C\in \R^{N\times d}:\Gamma_{1,\infty,1,1}(C)\le k\} . $$

By the definition $Q_{1,\infty,1,1}$ satisfies condition~\eqref{eqn:feasibility}. See section \ref{subsec:rademacher} for definition of Ramemacher width. 

\begin{lemma}\label{lem:rademacher_unrelaxed}
It holds that
\begin{equation}
\SRC_m(Q_{1,\infty,1,1}) \le \widetilde{O}\left(\sqrt{\frac{k^2N}{m}}\right)\nonumber\mper
\end{equation}
\end{lemma}
\begin{proof}[Proof of Lemma~\ref{lem:rademacher_unrelaxed}]
Recall the definition of the sample set $\Omega$ of coordinates from $C$, and their multiplication by i.i.d Rademacher variables in section \ref{subsec:rademacher}. 
Reusing notation, let $\xi = \sigma \odot \Omega$ and we use $Q$ as a shorthand for $Q_{1,\infty,1,1}$. Here $\odot$ means the entry-wise Hadamard product (namely, each coordinate in $\Omega$ is multiplied by an independent Rademacher random variable). We have that 
	\begin{align}
	\SRC_m(Q) &= \Exp_{\xi}\left[\frac{1}{m}\sup_{C \in Q}\inner{C,\xi}\right] \nonumber\\
	& = \Exp_{\xi}\left[\frac{1}{m}\sup_{\|A\|_{\ell_1\to \ell_{\infty}}\le 1, \|B\|_{\ell_1\to \ell_{1}\le k} }\inner{AB,\xi}\right] \tag{by defintiion of $Q$}\\
	& = \Exp_{\xi}\left[\frac{1}{m}\sup_{\|A\|_{\ell_1\to \ell_\infty}\le 1, \|B\|_{\ell_1\to \ell_{1}\le k} }\inner{B,A^{\top}\xi}\right] \nonumber\\
	& \le \Exp_{\xi}\left[\frac{1}{m}\sup_{\|A\|_{\ell_1\to \ell_\infty}\le 1} k \sum_{i=1}^N\|A^{\top}\xi_i\|_{{\infty}}\right] \tag{By $\inner{U,V}\le \left(\sum_{i=1}^N \|U_i\|_{\infty}\right)\|V\|_{\ell_1\to \ell_{1}}$}\mper
	\end{align}	Let $\rho = \frac{m}{dN}$ be the probability of any entry belongs to $\Omega$. Let $\xi_i$ denote the $i$-th column of $\xi$, and $A_j$ denotes the $j$-th column of $A$. Therefore, each entry of $\xi_i$ has probability $\rho/2$ to be +1 and -1, and has probability $1-\rho$ of being 0.  By 	concentration inequality we have that for $\rho \ge \frac{\log r}{d}$, and any fixed $A$ with $\|A\|_{\ell_1\to \ell_\infty} = \max \|A_j\|_\infty \le 1$, 
	\begin{equation}
	\Exp_{\xi_i}[\|A^{\top}\xi_i\|_{\infty}] \le O(\sqrt{\rho d}\log r\log d)\mper
	\end{equation}
		Moreover, we have that 
		\begin{equation}
		\Var_{\xi_i}[\|A^{\top}\xi_i\|_{\infty}] \le O(\sqrt{\rho d}\log r\log d)\mper
		\end{equation}
	Moreover, $\|A^{\top}\xi_i\|_{\infty}$ has an sub-exponential tail. (Technically, its $\psi_{1}$-Orlicz norm is bounded by  $O(\sqrt{\rho d}\log r\log d)$). 
	Note that the variance of $\sum_{i=1}^N \|A^{\top}\xi_i\|_{\infty}$ will decrease as $N$ increases, and therefore for large enough $N = (dr\rho)^{\Omega(1)}$, we will have that with probability $1-\exp(-\Omega(dr)^{\Omega(1)})$, $$	\sum_{i=1}^N \|A^{\top}\xi_i\|_{\infty} \le O(N\sqrt{\rho d\log r\log d})$$
Therefore, using the standard $\epsilon$-net covering argument, we obtain that with high probability, 
$$\sup_{\|A\|_{\ell_1\to \ell_2}\le \sqrt{d}}  	\sum_{i=1}^N \|A^{\top}\xi_i\|_{\infty} \le O(N\sqrt{\rho d\log r\log d})\mper $$

Hence, altogether we have
	\begin{align}
	\SRC_m(Q) 	& \le \Exp_{\xi}\left[\frac{1}{m}\sup_{\|A\|_{\ell_1\to \ell_2}\le \sqrt{d}} k \|A^{\top}\xi\|_{\ell_1\to\ell_{\infty}}\right]  \le \widetilde{O}\left(\sqrt{\frac{k^2N}{m}}\right)\nonumber\mper 	\end{align}

	\end{proof}

\subsection{Convex Relaxation for $\Gamma_{1, \infty,1,1}$ norm}\label{subsec:relaxed_rade}

\subsubsection{Sum-of-squares relaxation} \label{subsec:prelim_sos}

Here we will only briefly introduce the basic ideas of Sum-of-Squares (Lasserre) relaxation~\cite{Parrilo00,lasserre01} that will be used for this paper. We refer readers to the extensive study~\cite{lasserre15, Laurent09, BarakS14} for detailed discussions of sum of squares proofs and their applications to algorithm design. Recently, there has been a popular line of research on applications of sum-of-squares algorithms to machine learning problems~\cite{BKS14,BarakKS14,DBLP:conf/colt/BarakM16, DBLP:conf/nips/MaW15,DBLP:conf/approx/GeM15, DBLP:conf/colt/HopkinsSS15,DBLP:conf/stoc/HopkinsSSS16,MSS16}. Here our technique in the next subsection is most related to that of~\cite{DBLP:conf/colt/BarakM16}, with the main difference that we deal with $\ell_1$ norm constraints that are not typically within the SoS framework.

Let $\polyring{d}$ denote the set of all real polynomials of degree at most $d$ with $n$ variables $x_1,\dots,x_n$. We start by defining the notion of {\em pseudo-expectation}. The intuition is that the pseudo-expectation behave like the actual expectation of a real probability distribution on squares of polynomials. 
\begin{definition} [pseudo-expectation]
	A degree-$d$ pseudo-expectation $\pE$ is a linear operator that maps $\polyring{d}$ to $\R$ and satisfies $\pE(1)= 1$ and $\pE(p^2(x)) \ge 0$ for all real polynomials $p(x)$ of degree at most $d/2$.
			\end{definition}

\begin{definition}
	Given a set of polynomial equations $\cA = \{q_1(x)=0,\dots, q_n(x) = 0\}$, we say degree-$d$ pseudo-expectation $\pE$ satisfies constraints $\cA$ if $\pE\left[q_i(x)r(x)\right] = 0$ for every $i$ and $r(x)$ such that $\deg(q_ir)\le d$. 
	\end{definition}

One can optimize over the set of pseudo-expectations that satisfy $\cA$ in $n^{O(d)}$ time by the following semidefinite program: 
\begin{align}
\textrm{Variables} \quad & \pE[x^S] \quad \quad \quad \quad\quad ~~~ \forall S: |S|\le d \nonumber\\\textrm{Subject to} \quad &\pE\left[q_i(x)x^K\right] = 0 \quad \quad  \forall i, K: |K|+\deg(q_i) \le d\nonumber \\
& \pE\left[x^{\otimes d/2}(x^{\otimes d/2})^{\top}\right] \succeq 0 \nonumber
\end{align}

\begin{definition} [SoS proof of degree $d$]
	For a set of constraints $\cA= \{q_1(x) = 0,\dots,q_{n}(x) = 0\}$, and an integer $d$, we write $$\cA \vdash_d p(x)\ge q(x)$$ if there exists polynomials $h_i(x)$ for $i = 0,1,\dots, \ell$ and $g_j(x)$ for $j = 1,\dots,t$ such that 	$\deg(h_i)\le d/2$  and $\deg(g_jr_j)\le d$ that satisfy
	$$p(x)-q(x) = \sum_{i=1}^{\ell}h_i(x)^2 + \sum_{j=1}^{t}r_j(x)g_j(x)\mper$$
	We will drop the subscript $d$ when it is clear form the context. \end{definition}

The following fact follows from the definitions above but will be useful throughout the proof. 

\begin{proposition}\label{prop:basic_sos}
	Suppose $\cA \vdash_d  p(x)\ge q(x)$. Then for any degree-$d$ pseudo-expectation $\pE$ that satisfies $\cA$, we have $\pE[p(x)]\ge \pE[q(x)]$. 
\end{proposition}

\subsubsection{Relaxation for $\Gamma_{1, \infty,1,1}$ norm}\label{subsec:relaxation}

In this section, we introduce convex relaxations for the $\Gamma_{1,\infty,1,1}$ norm. 
For convenience, let $p$ be a power of 2. Let $A$ and $B$ be formal variables of dimension $d\times r$ and $r\times N$ in this section. We introduce more formal variables for the relaxation. Let $b$ be formal variables of dimension $r\times N$.  We consider the following set of polynomial constraints over formal variables $A,B,b$: 
\begin{align}
\cA  &= \left\{\forall i,j, B_{ij} = b_{ij}^{p-1}, \sum_{\ell=1}^{r} b_{\ell j}^p \le k^{p/(p-1)},  \forall i,k, A_{ik}^2 \le 1\right\}\nonumber\mper
\end{align}
 For any \textit{real} matrix $C\in \R^{d\times N}$, we define
\begin{align}
\cA(C) = \{C = AB\}\cup \cA\nonumber\mper
\end{align}

We define our convex relaxation for $Q_{1,\infty,1,1}$ as follows,  \begin{align}Q^{\sos}_{p} = \{C\in \R^{d\times N}: \exists \textup{degree-$O(p^2)$ pseudo-expectation}~\pE \textup{ that satisfies } \cA(C) \}\label{eqn:10}\end{align}

\begin{lemma}
	For any $p\ge 4$, we have
$$Q_{1,\infty,1,1} \subset Q^{\sos}_{p}$$
and therefore $Q^{\sos}_p$ satisfies condition~\eqref{eqn:feasibility}.
\end{lemma}

\begin{proof}
	Suppose  $C\in Q_{1,\infty,1,1}$. Then by definition of the $\Gamma_{1,\infty,1,1}$-norm (equation~\eqref{eqn:def-factorial-norm}), we have that there exists matrices $U,V$ such that $C = UV$ and $U_{ij}^2\le 1$ and $\norm{V}_{\ell_1\to\ell_1}\le k$. Now construct $v_{ij} = V_{ij}^{1/(p-1)}$. We have $\sum_{\ell=1}^r b_{ij}^p \le \left(\sum_{\ell=1}^{r}v_{ij}^{p-1}\right)^{p/(p-1)}\le k^{p/(p-1)} $. Therefore, Then we have that $A =U,B=V,b=v$ satisfies the constraint~\eqref{eqn:10}.   Then the trivial pseudo-expectation operator $\pE[p(A,B,b)] = p(U,V,v)$ satisfies $\mathcal{A}(C)$ and therefore $C\in Q^{\sos}_p$ by definition.  
\end{proof}

\begin{theorem}\label{thm:src_relaxation}
	Suppose $N =  d^{c_0p}$ for large enough absolute constant $c_0$. Then the sampling Rademacher complexity of $Q^{\sos}_p$ is bounded by, 
	\begin{equation}
		\SRC_m(Q^{\sos}_p)  \le O\left(\sqrt{\frac{p^2Nk^2 r^{2/p}\log d}{m}}\right)\mper\nonumber
	\end{equation}
\end{theorem}

The proof of Theorem~\ref{thm:src_relaxation} finally boils down to prove certain empirical process statement with SoS proofs. We start with the following two lemmas. 

\begin{lemma}\label{lem:sos_proof}
Suppose $N =  d^{c_0p}$ for larger enough constant $c_0$, and let $\xi = \sigma \odot \Omega$ where $\sigma$ and $\Omega$ are defined in Section~\ref{subsec:rademacher}. Then, we have 	\begin{equation}
	\cA  \vdash \inner{AB,\xi}^p\le  N^{p-1} k^{p} \sum_{i=1}^N \|A^{\top}\xi_i\|_p^p \nonumber
	\end{equation}
\end{lemma}

\begin{proof}
	Let $\xi_i$ be the $i$-th column of $\xi$. We have that 
	\begin{align}
		\cA  \vdash & \inner{AB,\xi}^p = \inner{A^{\top}\xi, B}^p = \left(\sum_{i=1}^N \inner{A^{\top}\xi, B_i}\right)^p \nonumber\\
		& \le N^{p-1} \left(\sum_{i=1}^{N} \inner{A^{\top}\xi_i, B_i}^p\right) \tag{since $\vdash \left(\frac{1}{N}\sum_{i\in N} \alpha_i\right)^p \le \frac{1}{N}\sum_{i\in [N]} \alpha_i^p$ }\\
		& = N^{p-1}\sum_{i=1}^N \inner{A^{\top}\xi_i, b_i^{\odot p-1}}^p \tag{by $B_{ij} = b_{ij}^{p-1}$}\\
		& \le  N^{p-1} \sum_{i=1}^N \|A^{\top}\xi_i\|_p^p \|b_i\|_p^{(p-1)p} \tag{by $\vdash \left(\sum_{j}a_ib_j^{p-1}\right)^p \le \left(\sum_i a_i^p\right)\left(\sum b_i^p\right)^{p-1}$; see Lemma~\ref{lem:sos_holder}}\\
		& \le N^{p-1} k^{p} \sum_{i=1}^N \|A^{\top}\xi_i\|_p^p \tag{by constraint $\sum_{\ell=1}^{r} b_{\ell j}^p \le k^{p/(p-1)}$}
			\end{align}
\end{proof}

\begin{lemma}\label{lem:sos_concentration}
		In the setting of Lemma~\ref{lem:sos_proof}, let $\rho = m/(Nd)$. Let $x = (x_1,\dots, x_d)$ be indeterminates and let $\mathcal{B} = \{x_1^2\le 1, \dots, x_d^2\le 1\}$.  Suppose $\rho \ge 1/d$.  With probability at least $1-\exp(-\Omega(d))$ over the choice of $\xi$,
	there exists a  SoS proof, 
	\begin{equation}
		\mathcal{B}\vdash \|\xi^{\top}x\|_p^p \le N\cdot O(\rho d p^2)^{p/2}. \label{eqn:eqn18}
	\end{equation}
	As an immediate consequence, let $\xi_i$ be the $i$-th column of $\xi$, then, 
	 with probability at least $1-\exp(-\Omega(d))$ over the choice of $\xi$ there exists SoS proof, 
	\begin{equation}
	\cA \vdash \sum_{i=1}^N \|A^{\top}\xi_i\|_p^p\le  Nr \cdot O(\rho d p^2)^{p/2}\mper\label{eqn:19}
	\end{equation}
		\end{lemma}

\begin{proof}
	Recall that $p$ is an even number. We have that 
	\begin{align}
	 \mathcal{B}\vdash \|\xi^{\top}x\|_p^p & =  \left(x^{\otimes p/2}\right)^{\top}\left(\sum_{i=1}^N \xi_i^{\otimes p/2} \left(\xi_i^{\otimes p/2}\right)^{\top}\right)x^{\otimes p/2} \nonumber
	\end{align}
	Let $T = \sum_{i=1}^N \xi_i^{\otimes p/2} \left(\xi_i^{\otimes p/2}\right)^{\top} \in \R^{d^{p/2}\times d^{p/2}}$. Then, by definition we have that $\Exp\left[\left(x\tp{p/2}\right)^{\top}  Tx\tp{p/2}\right] = \Exp\left[\sum_{i=1}^{N}\inner{\xi_i, x}^p\right]\nonumber$. It follows that 
		\begin{align}
		\mathcal{B} \vdash \|\xi^{\top}x\|_p^p &= \left(x\tp{p/2}\right)^{\top}  (T-\Exp[T])x\tp{p/2} + \Exp\left[\sum_{i=1}^{N}\inner{\xi_i, x}^p\right]\nonumber\\
		&\le \|x\|^p \Norm{T-\Exp[T]} + \Exp\left[\sum_{i=1}^{N}\inner{\xi_i, x}^p\right]\tag{by $\vdash y^{\top}By \le \|y\|^2 \|B\|$}\\
		&\le d^{p/2} \Norm{T-\Exp[T]} + \Exp\left[\sum_{i=1}^{N}\inner{\xi_i, x}^p\right]\nonumber\mper
		\end{align}
	Let $\zeta$ have the same distribution as $\xi_i$.  Then we have that 	\begin{align}
	 \Exp\left[\inner{\zeta, x}^p\right] = \sum_{\alpha} t_{\alpha}x^{\alpha}\mper\nonumber 	\end{align}
	where $t_{\alpha} = 0$ if $x^\alpha$ contains an odd degree. Moreover, let $|\alpha|$ be the number $x_i$ with non-zero exponents that $x^{\alpha}$ has. Then we have $t_{\alpha} \le  \rho^{|\alpha|} p^{p}$. Therefore we have that for $\rho \ge 1/d$, 
	\begin{align}
		\mathcal{B}\vdash \Exp\left[\inner{\zeta, x}^p\right] \le d^{p/2}\rho^{p/2}p^p \mper\nonumber
	\end{align}
		It follows that 
	\begin{align}
	\mathcal{B}	\vdash \Exp\left[\sum_{i=1}^N\inner{\xi_i, x}^p\right]  \le N\cdot O(\rho d p^2)^{p/2}\nonumber	\mper\end{align}
	Next, we use Bernstein inequality to bound $\Norm{T-\Exp[T]}$. We control the variance by 
	\begin{align}
	\sigma^2 \triangleq \Norm{\Exp\left[\sum_{i=1}^N \|\xi\|^p \xi_i^{\otimes p/2} \left(\xi_i^{\otimes p/2}\right)^{\top} \right]} \le Nd^{2p}. \nonumber 
	\end{align}
	Moreover, each summand in the definition of $T$ can be bounded by $\Norm{ \xi_i^{\otimes p/2} \left(\xi_i^{\otimes p/2}\right)^{\top} } \le d^p$. Note that these bounds are not necessarily tight since they already suffice for our purpose. Therefore, we obtain that with high probability over the chance of $\xi$, it holds that $\Norm{T-\Exp[T]} \le \sqrt{Nd^{O(p)} \log N}$.  Therefore for $N \ge d^{\Omega(p)}$, we have that with high probability over the choice of $\xi$, 
		\begin{align}
		\mathcal{B} \vdash \|\xi^{\top}x\|_p^p \\		&\le d^{p/2} \Norm{T-\Exp[T]} + \Exp\left[\sum_{i=1}^{N}\inner{\xi_i, x}^p\right]\le N\cdot O(\rho d p^2)^{p/2}\mper \nonumber
		\end{align}
	Finally we show equation~\eqref{eqn:eqn18} implies equation~\eqref{eqn:19}. 	Let $A_i$ be the $i$-th column of $A$. Note that we have $\sum_{i=1}^N \|A^{\top}\xi_i\|_p^p  = \sum_{i=1}^{r} \|\xi^{\top}A_i\|_p^p$, since both left-hand side and right-hand side are equal to the $p$-th power of all the entries of the matrix $A^{\top}\xi$. Since $\cA \vdash \{\forall j, A_{ij}^2\le 1\}$, we can invoke the first part of the Lemma and use equation~\eqref{eqn:eqn18} to complete the proof. 
	\end{proof}
Combining Lemma~\ref{lem:sos_proof} and Lemma~\ref{lem:sos_concentration} we obtain the following corollary,
\begin{corollary}\label{cor:sos_proof}
	In the setting of Lemma~\ref{lem:sos_proof}, with probability $1-\exp(-\Omega(d))$, we have
	\begin{equation}
	\cA  \vdash \inner{AB,\xi}^p\le  N^{p} k^{p} r \cdot O(\rho d)^{p/2}\mper\label{eqn:1}
	\end{equation}
\end{corollary}
Now we are ready to prove Theorem~\ref{thm:src_relaxation}. Essentially the only thing that we need to do is take pseudo-expectation on the both sides of equation~\eqref{eqn:1}. 
\begin{proof}[Proof of Theorem~\ref{thm:src_relaxation}]
		Recall that $\xi = \sigma \odot \Omega$. We have that 
		\begin{align}
		\SRC_m(Q^{\sos}_p)^p &= \Exp_{\xi}\left[\frac{1}{m^p}\sup_{C \in Q^{\sos}_d}\inner{C,\xi}\right]^p \nonumber\\
		& \le \Exp_{\xi}\left[\frac{1}{m^p}\sup_{C \in Q^{\sos}_d}\inner{C,\xi}^p\right] \tag{by Jensen's inequality}\\
		& = \Exp_{\xi}\left[\frac{1}{m^p}\sup_{\pE \textup{ that satisfies }\cA(C)}\pE\left[\inner{AB,\xi}\right]\right] \tag{by definition of $Q^{\sos}_p$}\\
		& \le \Exp\left[\frac{1}{m^p}\sup_{\pE \textup{ that satisfies }\cA(C)}\pE\left[\inner{AB,\xi}\right] \mid \textup{equation~\eqref{eqn:19} holds}\right] \nonumber\\
		& + \Pr\left[\textup{equation~\eqref{eqn:19} doesn't hold}\right]\cdot d^{O(1)} \nonumber\\
		& \le \frac{1}{m^p}N^{p}k^{p} r \cdot O(\rho d p^2)^{p/2} + \exp(-\Omega(d)) d^{O(1)}\tag{by Corollary~\ref{cor:sos_proof} and Proposition~\ref{prop:basic_sos}} 		\\
		& \le \frac{1}{m^p}N^{p}k^{p} r \cdot O(\rho d p^2)^{p/2}\mper\nonumber
		\end{align}
		Therefore using the fact that $\rho = m/(Nd)$, taking $p$-th root on the equation above, we obtain, 
		\begin{align}
	\SRC_m(Q^{\sos}_p)\le O\left(\sqrt{\frac{p^2Nr^{2/p}k^2 \log d}{m}}\right) \mper \nonumber
		\end{align}
	\end{proof}

\subsection{Proof of Theorem~\ref{thm:single}}\label{subsec:single}
\begin{algorithm}\caption{Encoding/decoding for improper dictionary learning} \label{alg:improper_dict_single}
	{\bf Given: } Integer $N$. Convex set $Q\subset \R^{d\times N}$. Sampling probability $\rho$. 
	\begin{enumerate}
		\item {\bf Encoding: }  {\bf input: } data point $x \in \R^{d}\sim \cD$, {\bf output: } codeword $y = h(x) \in \R^{k\times N}$ 		
		Draw $N-1$ data points $X_1\in \R^{d\times N-1}\sim \cD^{N-1}$. 
		
		\textit{Denoising step: }
				\begin{equation}
		[Z_1,z] = 	\argmin_{[C,c]\in Q} \gNorm{[X_1,x]-[C,c]}\mcom \label{eqn:eqn19} 		\end{equation}
		\textit{Random sampling step:}
		\begin{equation}
		h(x) = \Po(z)\nonumber \mcom
		\end{equation}
		where $\Po(z)$ is a random sampling of $z$ where each entry is picked with probability $\rho$. 
		\item {\bf Decoding: }		{\bf input: } codeword $y\in \R^{k\times N}$ 
		{\bf output: } reconstruction $g(y)\in \R^{d\times N}$ 
		
		Take $N$ more samples $x_1,\dots, x_N$. Encode them to $y_1 = h(x_1),\dots,y_N =h(x_N)$. Then, 
		compute 
		\begin{equation}
		[\tilde{z}_1,\dots, \tilde{z}_{N-1}, \tilde{z}] = \argmin_{C\in Q} \gNorm{\Po([C,c]) -[y_1,\dots, y_N, y]} 
				\mper \nonumber		\end{equation}
		Return $g(y) = \tilde{z}$. 
											\end{enumerate}
\end{algorithm}

In this subsection we prove Theorem~\ref{thm:single} using the machinery developed in previous subsections. Essentially the idea here is to reduce the problem of encoding one example to the problem of encoding a group of examples. To encode an example $x$, we draw $N-1$ fresh examples $x_1,\dots, x_{N-1}$, and then call the group encoding Algorithm~\ref{alg:improper_dict}. 

We also note that the encoding procedure can be simplified by removing the denoising step, and this still allows efficient decoding steps. Since in this case the encoding contains more noise from the data, we prefer the current scheme. 

\begin{proof}[Proof of Theorem~\ref{thm:single}]
	We will prove that Algorithm~\ref{alg:improper_dict_single} with $Q = Q_{\textup{sos}}^p$ and $\rho = O(k^2r^{2/p}d^{-1}/\delta^2 \cdot \log d)$ gives the desired encoding/decoding. Indeed with high probability, the encoding length is $\tilde{O}(\rho d) = \tilde{O}(k^2r^{2/p}/\delta^2)$. 
	Next we will show that  we obtain good reconstruction in the sense that $\frac 1 d \Exp\left[\gNorm{x-g(h(x))}\right] = \frac 1 d \Exp\left[\gNorm{x-\tilde{z}}\right] \le O(\epsilon^{\star} + \delta)$.  Here and throughout the proof, the expectation is over the randomness of $x$ and the randomness of the algorithm (that is, the randomness of $x, x_1,\dots, x_{N-1}$ and the random sampling). First of all, by symmetry, we have that 
	\begin{align}
	\Exp\left[\gNorm{x-\tilde{z}_1}\right] = \frac{1}{Nd} \Exp\left[\gNorm{[x_1,\dots, x_{N-1},x]- [\tilde{z}_1,\dots, \tilde{z}_{N-1},\tilde{z}]}\right]\mper\label{eqn:22}
	\end{align} 
	Let $X = [x_1,\dots, x_{N-1},x]$, and $\tilde{Z} = [\tilde{z}_1,\dots, \tilde{z}_{N-1},\tilde{z}]$ for simplicity of notations. Let $[Z_{1,i},z_i] = \argmin_{[C,c]\in Q} \gNorm{[X_{1,i}, x_i]- [C,c]}$ where $X_{1,i}$ be the samples drawn in the encoding process for $x_i$. Let $Z = [z_1,\dots, z_{N-1},z]$. By Lemma~\ref{lem:denoising} and the symmetry, we have that for any $i$, $\frac 1 d \Exp\left[\gNorm{z_i-x_i}\right] \le \frac{1}{dN}\Exp\left[\gNorm{[Z_{1,i},z_i]-[X_{1,i},x_i]}\right]\le \epsilon^{\star}$. Thus, we have $\frac{1}{dN}\Exp\left[\gNorm{X-Z}\right]\le \epsilon^{\star}$. 
	
	 Let $\tilde{Z}\in \R^{d\times N}$ be a helper matrix in the analysis defined as 
	\begin{align}
	\hat{Z} = \argmin_{C\in Q} \gNorm{X-C}\mper
	\end{align}
	Then by Lemma~\ref{lem:denoising} we have that $\frac{1}{Nd}\Exp\left[\gNorm{\hat{Z}-X}\right]\le \epsilon^{\star}$. Then by triangle inequality we obtain that \begin{align}
	\frac{1}{Nd}\Exp\left[\gNorm{\hat{Z}-Z}\right]\le 2\epsilon^{\star}\label{eqn:21}\mper
	\end{align}
	
	Note that by definition $\hat{Z}\in Q$. Therefore,
	 by Lemma~\ref{lem:src}, we have that 
	 \begin{align}
	 \frac{1}{Nd} \Exp\left[\gNorm{\tilde{Z}-Z}\right] & \le \frac{1}{m} \Exp\left[\gNorm{\Po(\tilde{Z})-\Po(Z)} \right] +  2\SRC_m(W) \tag{by Lemma~\ref{lem:src}}\\
	 &\le \frac{1}{m} \Exp\left[\gNorm{\Po(\hat{Z})-\Po(Z)} \right] +  2\SRC_m(W) \tag{since $\tilde{Z}$ is the minimizer, and $\hat{Z}\in Q$. }\\
	 	 &=\frac{1}{Nd} \Exp\left[\gNorm{\hat{Z}-Z} \right] +  2\SRC_m(W) \nonumber \\
	 	 & \le \epsilon^{\star} + \delta \mper\tag{using equation~\eqref{eqn:21} and Theorem~\ref{thm:sos} (c) }	 \end{align}
	 Finally by triangle inequality again we have 
	 \begin{align}
	 	 \frac{1}{Nd} \Exp\left[\gNorm{\tilde{Z}-X}\right] \le 	 \frac{1}{Nd} \Exp\left[\gNorm{\tilde{Z}-Z}\right] + 	 \frac{1}{Nd} \Exp\left[\gNorm{X-Z}\right] \le \epsilon^{\star}\mper\label{eqn:23}
	 \end{align}
	 Combining equation~\eqref{eqn:22} and~\eqref{eqn:23} we complete the proof. 
\end{proof}

\section{Conclusions} 
We have defined a new  framework for unsupervised learning which replaces generative assumptions by notions of  reconstruction error and encoding length. This framework is comparative, and allows learning of particular hypothesis classes with respect to an unknown distribution by other hypothesis classes. 

We demonstrate its usefulness by giving new polynomial time algorithms for two unsupervised hypothesis classes. First, we give new polynomial time algorithms for dictionary models in significantly broader range of parameters and assumptions.  Another domain is the class of spectral encodings, for which we consider a new class of models that is shown to strictly encompass PCA and kernel-PCA. This new class is capable, in contrast to previous spectral models, learn algebraic manifolds. We give efficient learning algorithms for this class based on convex relaxations.

\subsection*{Acknowledgements}

We thank Sanjeev Arora for many illuminating discussions and crucial observations  in earlier phases of this work, amongst them that a representation which preserves information for all classifiers requires lossless compression.

\bibliographystyle{alpha}
\bibliography{ref,ref2}

\newcommand{\etalchar}[1]{$^{#1}$}
\begin{thebibliography}{PWMH13}

\bibitem[ADJ{\etalchar{+}}13]{DBLP:conf/isit/AcharyaDJOS13}
Jayadev Acharya, Hirakendu Das, Ashkan Jafarpour, Alon Orlitsky, and
  Ananda~Theertha Suresh.
\newblock Tight bounds for universal compression of large alphabets.
\newblock In {\em Proceedings of the 2013 {IEEE} International Symposium on
  Information Theory, Istanbul, Turkey, July 7-12, 2013}, pages 2875--2879,
  2013.

\bibitem[AEB05]{Aharon05k-svd-design}
Michal Aharon, Michael Elad, and Alfred Bruckstein.
\newblock K-svd: Design of dictionaries for sparse representation.
\newblock In {\em IN: PROCEEDINGS OF SPARS’05}, pages 9--12, 2005.

\bibitem[AGH{\etalchar{+}}14]{Anandkumar2014}
Animashree Anandkumar, Rong Ge, Daniel Hsu, Sham~M. Kakade, and Matus
  Telgarsky.
\newblock Tensor decompositions for learning latent variable models.
\newblock {\em J. Mach. Learn. Res.}, 15(1):2773--2832, January 2014.

\bibitem[AGM12]{arora2012learning}
Sanjeev Arora, Rong Ge, and Ankur Moitra.
\newblock Learning topic models--going beyond svd.
\newblock In {\em Foundations of Computer Science (FOCS), 2012 IEEE 53rd Annual
  Symposium on}, pages 1--10. IEEE, 2012.

\bibitem[AGM13]{arora2013new}
Sanjeev Arora, Rong Ge, and Ankur Moitra.
\newblock New algorithms for learning incoherent and overcomplete dictionaries.
\newblock {\em arXiv preprint arXiv:1308.6273}, 2013.

\bibitem[AGMM15]{DBLP:conf/colt/AroraGMM15}
Sanjeev Arora, Rong Ge, Tengyu Ma, and Ankur Moitra.
\newblock Simple, efficient, and neural algorithms for sparse coding.
\newblock In {\em Proceedings of The 28th Conference on Learning Theory, {COLT}
  2015, Paris, France, July 3-6, 2015}, pages 113--149, 2015.

\bibitem[BBV08]{Balcan:2008:DFC:1374376.1374474}
Maria-Florina Balcan, Avrim Blum, and Santosh Vempala.
\newblock A discriminative framework for clustering via similarity functions.
\newblock In {\em Proceedings of the Fortieth Annual ACM Symposium on Theory of
  Computing}, STOC '08, pages 671--680, 2008.

\bibitem[BDA09]{benda09}
Shai Ben-David and Margareta Ackerman.
\newblock Measures of clustering quality: A working set of axioms for
  clustering.
\newblock In D.~Koller, D.~Schuurmans, Y.~Bengio, and L.~Bottou, editors, {\em
  Advances in Neural Information Processing Systems 21}, pages 121--128. Curran
  Associates, Inc., 2009.

\bibitem[Ber71]{berger1971rate}
Toby Berger.
\newblock Rate distortion theory: A mathematical basis for data compression.
\newblock 1971.

\bibitem[BKS14]{BarakKS14}
Boaz Barak, Jonathan~A. Kelner, and David Steurer.
\newblock Rounding sum-of-squares relaxations.
\newblock In {\em STOC}, pages 31--40, 2014.

\bibitem[BKS15a]{DBLP:conf/stoc/BarakKS15}
Boaz Barak, Jonathan~A. Kelner, and David Steurer.
\newblock Dictionary learning and tensor decomposition via the sum-of-squares
  method.
\newblock In {\em Proceedings of the Forty-Seventh Annual {ACM} on Symposium on
  Theory of Computing, {STOC} 2015, Portland, OR, USA, June 14-17, 2015}, pages
  143--151, 2015.

\bibitem[BKS15b]{BKS14}
Boaz Barak, Jonathan~A. Kelner, and David Steurer.
\newblock Dictionary learning and tensor decomposition via the sum-of-squares
  method.
\newblock In {\em Proceedings of the Forty-seventh Annual ACM Symposium on
  Theory of Computing}, STOC '15, 2015.

\bibitem[BM15]{DBLP:journals/corr/BarakM15}
Boaz Barak and Ankur Moitra.
\newblock Tensor prediction, rademacher complexity and random 3-xor.
\newblock {\em CoRR}, abs/1501.06521, 2015.

\bibitem[BM16]{DBLP:conf/colt/BarakM16}
Boaz Barak and Ankur Moitra.
\newblock Noisy tensor completion via the sum-of-squares hierarchy.
\newblock In {\em Proceedings of the 29th Conference on Learning Theory, {COLT}
  2016, New York, USA, June 23-26, 2016}, pages 417--445, 2016.

\bibitem[BNJ03]{Blei2003}
David~M. Blei, Andrew~Y. Ng, and Michael~I. Jordan.
\newblock Latent dirichlet allocation.
\newblock {\em J. Mach. Learn. Res.}, 3:993--1022, March 2003.

\bibitem[BS14]{BarakS14}
Boaz Barak and David Steurer.
\newblock Sum-of-squares proofs and the quest toward optimal algorithms.
\newblock In {\em Proceedings of International Congress of Mathematicians
  (ICM)}, 2014.
\newblock To appear.

\bibitem[CT06]{Cover:2006}
Thomas~M. Cover and Joy~A. Thomas.
\newblock {\em Elements of Information Theory (Wiley Series in
  Telecommunications and Signal Processing)}.
\newblock Wiley-Interscience, 2006.

\bibitem[DH06]{Donoho2006}
D.~L. Donoho and X.~Huo.
\newblock Uncertainty principles and ideal atomic decomposition.
\newblock {\em IEEE Trans. Inf. Theor.}, 47(7):2845--2862, September 2006.

\bibitem[GM15]{DBLP:conf/approx/GeM15}
Rong Ge and Tengyu Ma.
\newblock Decomposing overcomplete 3rd order tensors using sum-of-squares
  algorithms.
\newblock In {\em Approximation, Randomization, and Combinatorial Optimization.
  Algorithms and Techniques, {APPROX/RANDOM} 2015, August 24-26, 2015,
  Princeton, NJ, {USA}}, pages 829--849, 2015.

\bibitem[HK12]{DBLP:conf/icml/HazanK12}
Elad Hazan and Satyen Kale.
\newblock Projection-free online learning.
\newblock In {\em Proceedings of the 29th International Conference on Machine
  Learning, {ICML} 2012, Edinburgh, Scotland, UK, June 26 - July 1, 2012},
  2012.

\bibitem[HK13]{hsu2013learning}
Daniel Hsu and Sham~M Kakade.
\newblock Learning mixtures of spherical gaussians: moment methods and spectral
  decompositions.
\newblock In {\em Proceedings of the 4th conference on Innovations in
  Theoretical Computer Science}, pages 11--20. ACM, 2013.

\bibitem[HSS15]{DBLP:conf/colt/HopkinsSS15}
Samuel~B. Hopkins, Jonathan Shi, and David Steurer.
\newblock Tensor principal component analysis via sum-of-square proofs.
\newblock In {\em Proceedings of The 28th Conference on Learning Theory, {COLT}
  2015, Paris, France, July 3-6, 2015}, pages 956--1006, 2015.

\bibitem[HSSS16]{DBLP:conf/stoc/HopkinsSSS16}
Samuel~B. Hopkins, Tselil Schramm, Jonathan Shi, and David Steurer.
\newblock Fast spectral algorithms from sum-of-squares proofs: tensor
  decomposition and planted sparse vectors.
\newblock In {\em Proceedings of the 48th Annual {ACM} {SIGACT} Symposium on
  Theory of Computing, {STOC} 2016, Cambridge, MA, USA, June 18-21, 2016},
  pages 178--191, 2016.

\bibitem[JOS05]{DBLP:journals/tcs/JevticOS05}
Nikola Jevtic, Alon Orlitsky, and Narayana~P. Santhanam.
\newblock A lower bound on compression of unknown alphabets.
\newblock {\em Theor. Comput. Sci.}, 332(1-3):293--311, 2005.

\bibitem[Kle03]{kleinberg03}
Jon~M. Kleinberg.
\newblock An impossibility theorem for clustering.
\newblock In S.~Becker, S.~Thrun, and K.~Obermayer, editors, {\em Advances in
  Neural Information Processing Systems 15}, pages 463--470. MIT Press, 2003.

\bibitem[Las01]{lasserre01}
Jean~B. Lasserre.
\newblock Global optimization with polynomials and the problem of moments.
\newblock {\em SIAM Journal on Optimization}, 11(3):796--817, 2001.

\bibitem[Las15]{lasserre15}
Jean~Bernard Lasserre.
\newblock {\em {An introduction to polynomial and semi-algebraic
  optimization.}}
\newblock {Cambridge Texts in Applied Mathematics. Cambridge: Cambridge
  University Press. }, 2015.

\bibitem[Lau09]{Laurent09}
Monique Laurent.
\newblock Sums of squares, moment matrices and optimization over polynomials.
\newblock In Mihai Putinar and Seth Sullivant, editors, {\em Emerging
  Applications of Algebraic Geometry}, volume 149 of {\em The IMA Volumes in
  Mathematics and its Applications}, pages 157--270. Springer New York, 2009.

\bibitem[LLS{\etalchar{+}}13]{DBLP:conf/icml/LivniLSNSG13}
Roi Livni, David Lehavi, Sagi Schein, Hila Nachlieli, Shai Shalev{-}Shwartz,
  and Amir Globerson.
\newblock Vanishing component analysis.
\newblock In {\em Proceedings of the 30th International Conference on Machine
  Learning, {ICML} 2013, Atlanta, GA, USA, 16-21 June 2013}, pages 597--605,
  2013.

\bibitem[LMSS07]{DBLP:journals/combinatorica/LinialMSS07}
Nati Linial, Shahar Mendelson, Gideon Schechtman, and Adi Shraibman.
\newblock Complexity measures of sign matrices.
\newblock {\em Combinatorica}, 27(4):439--463, 2007.

\bibitem[MRT12]{mohri2012foundations}
Mehryar Mohri, Afshin Rostamizadeh, and Ameet Talwalkar.
\newblock {\em Foundations of machine learning}.
\newblock MIT press, 2012.

\bibitem[MSS16]{MSS16}
Tengyu Ma, Jonathan Shi, and David Steurer.
\newblock Polynomial-time tensor decompositions with sum-of-squares.
\newblock In {\em FOCS 2016, to appear}, 2016.

\bibitem[MW15]{DBLP:conf/nips/MaW15}
Tengyu Ma and Avi Wigderson.
\newblock Sum-of-squares lower bounds for sparse {PCA}.
\newblock In {\em Advances in Neural Information Processing Systems 28: Annual
  Conference on Neural Information Processing Systems 2015, December 7-12,
  2015, Montreal, Quebec, Canada}, pages 1612--1620, 2015.

\bibitem[Nat95]{DBLP:journals/siamcomp/Natarajan95}
B.~K. Natarajan.
\newblock Sparse approximate solutions to linear systems.
\newblock {\em {SIAM} J. Comput.}, 24(2):227--234, 1995.

\bibitem[OSZ04]{DBLP:journals/tit/OrlitskySZ04}
Alon Orlitsky, Narayana~P. Santhanam, and Junan Zhang.
\newblock Universal compression of memoryless sources over unknown alphabets.
\newblock {\em {IEEE} Trans. Information Theory}, 50(7):1469--1481, 2004.

\bibitem[Par00]{Parrilo00}
Pablo~A. Parrilo.
\newblock {\em Structured Semidefinite Programs and Semialgebraic Geometry
  Methods in Robustness and Optimization}.
\newblock PhD thesis, California Institute of Technology, 2000.

\bibitem[PWMH13]{paskov2013compressive}
Hristo~S Paskov, Robert West, John~C Mitchell, and Trevor Hastie.
\newblock Compressive feature learning.
\newblock In {\em Advances in Neural Information Processing Systems}, pages
  2931--2939, 2013.

\bibitem[Ris78]{rissanen1978modeling}
Jorma Rissanen.
\newblock Modeling by shortest data description.
\newblock {\em Automatica}, 14(5):465--471, 1978.

\bibitem[Sal09]{Salakhutdinov2009}
Ruslan Salakhutdinov.
\newblock {\em Learning Deep Generative Models}.
\newblock PhD thesis, University of Toronto, Toronto, Ont., Canada, Canada,
  2009.
\newblock AAINR61080.

\bibitem[SS05]{DBLP:conf/colt/SrebroS05}
Nathan Srebro and Adi Shraibman.
\newblock Rank, trace-norm and max-norm.
\newblock In {\em Learning Theory, 18th Annual Conference on Learning Theory,
  {COLT} 2005, Bertinoro, Italy, June 27-30, 2005, Proceedings}, pages
  545--560, 2005.

\bibitem[SST08]{Shamir2008}
Ohad Shamir, Sivan Sabato, and Naftali Tishby.
\newblock {\em Learning and Generalization with the Information Bottleneck},
  pages 92--107.
\newblock Springer Berlin Heidelberg, Berlin, Heidelberg, 2008.

\bibitem[Tib96]{10.2307/2346178}
Robert Tibshirani.
\newblock Regression shrinkage and selection via the lasso.
\newblock {\em Journal of the Royal Statistical Society. Series B
  (Methodological)}, 58(1):267--288, 1996.

\bibitem[TJ89]{tomczak1989banach}
N.~Tomczak-Jaegermann.
\newblock {\em Banach-Mazur distances and finite-dimensional operator ideals}.
\newblock Pitman monographs and surveys in pure and applied mathematics.
  Longman Scientific \& Technical, 1989.

\bibitem[TPB00]{IB-tishby}
Naftali Tishby, Fernando C.~N. Pereira, and William Bialek.
\newblock The information bottleneck method.
\newblock {\em CoRR}, physics/0004057, 2000.

\bibitem[Val84]{valiant1984theory}
Leslie~G Valiant.
\newblock A theory of the learnable.
\newblock {\em Communications of the ACM}, 27(11):1134--1142, 1984.

\bibitem[Vap98]{Vapnik1998}
Vladimir~N. Vapnik.
\newblock {\em Statistical Learning Theory}.
\newblock Wiley-Interscience, 1998.

\bibitem[Vu11]{vu-wedin}
Van Vu.
\newblock Singular vectors under random perturbation.
\newblock {\em Random Structures and Algorithms}, 39(4):526--538, 2011.

\end{thebibliography}

\appendix
	\newpage
\section{Proof of Theorem~\ref{thm:rademacher-generalization}}\label{subsec:generalization_proof}

\begin{proof}[Proof of Theorem~\ref{thm:rademacher-generalization}]
	\cite[Theorem 3.1]{mohri2012foundations} asserts that with probability at least $1-\delta$, we have that for every hypothesis $f \in \H$, 
	$$ \loss_{\D} (f)  \leq \loss_S( f)  + 2 \mathcal{R}_m(\H) + \sqrt{ \frac{\log \frac{1}{\delta}}{2 m }} $$
	
	by negating the loss function this gives
	$$ | \loss_{\D} (f)  - \loss_S( f )  | \leq  2\mathcal{R}_m(\H) + \sqrt{ \frac{\log \frac{2}{\delta}}{2 m }} $$
	and therefore, letting $f^* = \argmin_{f\in \H}\loss_{\D}(f)$, we have
	\begin{align}
	\loss_{\D} ( \hat{f}_{ERM})  & \leq \loss_S(  \hat{f}_{ERM}  )  + 2 \mathcal{R}_m(\H) + \sqrt{ \frac{\log \frac{1}{\delta}}{2 m }} & \tag{by ~\cite[Theorem 3.1]{mohri2012foundations} }\\
	& \leq \loss_S(  f^*  )  + 2 \mathcal{R}_m(\H) + \sqrt{ \frac{\log \frac{1}{\delta}}{2 m }} & \tag{ by definition of ERM } \\
	& \leq \loss_{\D} (  f^* )  + 6 \mathcal{R}_m(\H) + \sqrt{ \frac{ 4 \log \frac{1}{\delta}}{2 m }} & \tag{ using ~\cite[Theorem 3.1]{mohri2012foundations} again} 
	\end{align}
\end{proof}

\subsection{Low reconstruction error is sufficient for supervised learning} \label{appendix-reconstruction}
This section observes that low reconstruction error is a sufficient condition for unsupervised learning to allow supervised learning over any future task. 

\begin{lemma}
Consider any supervised learning problem with respect to distribution $\D$ over $\X \times \L$ that is agnostically PAC-learnable  with respect to $L$-Lipschitz loss function $\ell$ and with bias $\gamma_1$.  

Suppose that unsupervised hypothesis class $\H$  is \shmak-learnable with bias $\gamma_2$ over distribution $\D$ and domain $\X$, by hypothesis $f: \X \mapsto \Y$.  Then the distribution $\tilde{\D}_f$ over $ \Y \times \L$, which gives the pair $(f(x), y)$ the same measure as $\D$ gives to $(x,y)$,  is agnostically PAC-learnable with bias $\gamma_1 + L \gamma_2 $. 
\end{lemma}

\begin{proof}
Let $h : \X \mapsto \Y$ be a hypothesis that PAC-learns distribution $\D$. Consider the hypothesis 
$$ \tilde{h} : \Y \mapsto \L \ ,   \   \tilde{h}(y) = (h \circ g) (y) $$
Then by definition of reconstruction error and the Lipschitz property of $\ell$ we have
\begin{align}
\loss_{\tilde{\D}_f} (\tilde{h}) & =  \E_{(y,l) \sim \tilde{\D}_f} [ \ell( \tilde{h}(y) , l ) ] \nonumber\\
& =   \E_{(y,l) \sim \tilde{\D}_f} [ \ell( (h \circ g  ) (y) , l ) ] \nonumber\\
& =   \E_{(x,l) \sim {\D}} [ \ell( h (\tilde{x}) , l ) ] \tag{$\D(x) = \tilde{\D}_f(y)$ } \nonumber\\
& =   \E_{(x,l) \sim {\D}} [ \ell( h ({x}) , l ) ] +  \E_{(x,l) \sim {\D}} [ \ell( h (\tilde{x}) , l ) - \ell( h ({x}) , l ) ] \nonumber\\
& =   \gamma_1 + \E_{(x,l) \sim {\D}} [ \ell( h (\tilde{x}) , l ) - \ell( h ({x}) , l ) ]  \tag{ PAC learnability }  \\
& \leq \gamma_1 + L \E_{x \sim \D}  \| x - \tilde{x} \|  \tag{ Lipschitzness of $\ell \circ h$}  \\
& = \gamma_1 + L \E_{x \sim \D} \| x - g\circ f(x) \| \nonumber\\
& \leq \gamma_1 +  L \gamma_2 \tag{ \shmak-learnability} 
\end{align}
\end{proof}

\section{Proof of Theorem~\ref{thm:pca}}\label{sec:proof_pca}
\begin{proof}[Proof of Theorem~\ref{thm:pca}]
	We assume without loss of generality $\ss = 1$. For $\ss>1$ the proof will be identical since one can assume $x^{\otimes \ss}$ is the data points (and the dimension is raised to $d^{\ss}$).

	Let $x_1,\dots, x_m$ be a set of examples $\sim \cD^m$. It can be shown that any minimizer of ERM 
	\begin{equation}
	A^*= \argmin_{A\in \R^{d\times k}}	\|x_i- A^{\dagger} Ax_i\|^2 \label{eqn:erm}
	\end{equation}
	satisfies that $(A^*)^{\dagger}A^*$ is the the projection operator to the subspace of top $k$ eigenvector of $\sum_{i=1}^{m} x_ix_i^{\top}$. Therefore ERM~\eqref{eqn:erm} is efficiently solvable. 
	
	According to Theorem \ref{thm:rademacher-generalization}, the ERM hypothesis generalizes with rates governed by the Rademacher complexity of the hypothesis class. 		Thus, it remains to compute the Rademacher complexity of the hypothesis class for PCA. We assume for simplicity that all vectors in the domain have Euclidean norm bounded by one. 
	\begin{align} 
		\mathcal{R}_S(\H_{k}^{\pca}) & =   \E _{\sigma \sim \{ \pm 1\}^m } \left[ \sup_{(h,g) \in \H_{k}^{\pca}} \frac{1}{m} \sum_{i \in S} \sigma_i \ell((h,g) ,x_i ) \right]   \nonumber\\
		& =   \E _{\sigma \sim \{ \pm 1\}^m } \left[ \sup_{A\in \R^{d\times k}} \frac{1}{m} \sum_{i \in S} \sigma_i \|x_i -   A^\dagger A x_i\|^2  \right]   \nonumber\\
		& =   \E _{\sigma \sim \{ \pm 1\}^m } \left[ \sup_{A\in \R^{d\times k}} \frac{1}{m} \sum_{i \in S} \sigma_i  \trace((I- A^\dagger A)\left(\sum_{i=1}^{m}x_ix_i^{\top}\right)(I- A^\dagger A)^{\top}) \right]    \nonumber\\
		& =  \E _{\sigma \sim \{ \pm 1\}^m } \left[ \sup_{A\in R^{d \times k}}    \trace\left((I- A^\dagger A)\left(\frac{1}{m} \sum_{i=1}^{m}\sigma_ix_ix_i^{\top}\right)\right) \right]   \mper \nonumber
	\end{align}
	
	Then we apply Holder inequality, and effectively disentangle the part about $\sigma$ and $A$: 
	\begin{align} 
	&	\E _{\sigma \sim \{ \pm1 \}^m } \left[ \sup_{A\in \R^{d \times k}}    \trace\left((I- A^\dagger A)\left(\frac{1}{m} \sum_{i=1}^{m}\sigma_ix_ix_i^{\top}\right)\right) \right]  \nonumber\\
		& \le  \E _{\sigma \sim \{ \pm 1\}^m } \left[ \sup_{A\in \R^{d \times k}}   \|I- A^\dagger A\|_{F}\Norm{\frac{1}{m} \sum_{i=1}^{m}\sigma_ix_ix_i^{\top}}_F \right] \tag{Holder inequality} \\
		&  \le \sqrt{d}\E _{\sigma \sim \{ \pm 1\}^m } \left[ \Norm{\frac{1}{m} \sum_{i=1}^{m}\sigma_ix_ix_i^{\top}}_F \right]\tag{ since $A^{\dagger} A$ is a projection, $\|I-A^{\dagger}A\|\le 1$. } \\
		&  \le \sqrt{d}\E _{\sigma \sim \{ \pm 1\}^m } \left[ \Norm{\frac{1}{m} \sum_{i=1}^{m}\sigma_ix_ix_i^{\top}}_F^2 \right]^{1/2}\tag{Cauchy-Schwarz inequality } \\
		&\le \sqrt{d} \sqrt{\frac{1}{m^2}\sum_{i=1}^m \inner{\sigma_ix_ix_i^{\top}, \sigma_ix_ix_i^{\top}}}\tag{since $\E[\sigma_i\sigma_j] = 0$ for $i\neq j$}\\
		& \le \sqrt{d/m} \tag{by $\|x\|\le 1$}\mper	\end{align}
	Thus, from Theorem \ref{thm:rademacher-generalization} we can conclude that the class $\H_k^{\pca}$ is learnable with sample complexity $\tilde{O} ( 
	\frac{d}{\eps^2} )$\footnote{For $\ell > 1$ the sample complexity is $\tilde{O}(d^{\ell}/\eps^2)$.}.
\end{proof}

\subsection{Proof of Theorem~\ref{thm:main-sdp}} \label{sec:main_sdp_proof}
Theorem~\ref{thm:main-sdp} follows from the following lemmas and the generalization theorem~\ref{thm:rademacher-generalization} straightforwardly. 
\begin{lemma}\label{thn:spectral_decoding}
	Suppose distribution $\cD$ is $(k,\epsilon)$-regularly spectral decodable. Then for any $\delta > 0$, solving convex optimization~\eqref{eqn:objective_relaxation} with non-smooth Frank-Wolfe algorithm~\cite[Theorem 4.4]{DBLP:conf/icml/HazanK12} with $k' = O(\tau^4 k^4/\delta^4)$ steps 	gives a solution $\hat{R}$ of rank $k'$ such that $f(\hat{R})\le \delta+\epsilon$. 
	
\end{lemma}

\begin{lemma}\label{lem:2}
	The Rademacher complexity of the class of function $\Phi = \left\{\left\|Rx\tp{2} - x\tp{2}\right\|_{\tsop}:R \textup{ s.t. }\norm{R}_{S_1} \le \tau k\right\}$ with $m$ examples is bounded from above by at most $\mathcal{R}_m(\Phi)\le 2\tau k\cdot\sqrt{1/m}$
\end{lemma}

Here Lemma~\ref{lem:2} follows from the fact that $\left\|Rx\tp{2} - x\tp{2}\right\|_{\tsop}$ is bounded above by $2\tau k$ when $\|x\|\le 1 $ and $\|R\|_{S_1}\le \tau k$. The rest of the section focuses on the proof of Lemma~\ref{thn:spectral_decoding}.

Lemma~\ref{thn:spectral_decoding} basically follows from the fact that $f$ is Lipschitz and guarantees of the Frank-Wolfe algorithm. \begin{proposition}\label{prop:lipschitz}
	The objective function $f(R)$ is convex and 1-Lipschitz.  Concretely, 
	Let $\ell_x(R) = \norm{Rx\tp{2}-x\tp{2}}_{\tsop}$. Then 
	\begin{equation}
	\partial \ell_x \ni  (u\otimes v) (x\tp{2})^{\top}\nonumber
	\end{equation}
	where $\partial \ell_x$ is the set of sub-gradients of $\ell_x$ with respect to $R$, and $u,v\in \R^d$ are (one pair of) top left and right singular vectors of $\mat(Rx\tp{2}-x\tp{2})$. 
	\end{proposition}
\begin{proof}
	This simply follows from calculating gradient with chain rule. Here we use the fact that $A\in (\R^{d})^{\otimes 2}$, the sub-gradient of $\|A\|_{\op}$ contains the vector $a\otimes b$ where $a,b$ are the top left and right singular vectors of $\mat(A)$. We can also verify by definition that $(u\otimes v) (x\tp{2})^{\top}$ is a sub-gradient. 
	\begin{align}
	f(R') - f(R) &\ge (u\otimes v)^{\top} (R'x\tp{2} - R x\tp{2}) \tag{by convexity of $\|\cdot \|_{\op}$}\\
	& = \inner{(u\otimes v) (x\tp{2})^{\top}, R'-R}  \nonumber\mper
	\end{align}
\end{proof}

Now we are ready to prove Lemma~\ref{thn:spectral_decoding}.

\begin{proof}[Proof of Lemma~\ref{thn:spectral_decoding}]
	Since $\cD$ is $(k,\epsilon)$-regularly decodable, we know that there exists a rank-$k$ solution $R^*$ with $f(R^*)\le \epsilon$. Since $\norm{R^*}_{\op}\le \tau$, we have that $\norm{R}_{S_1}\le \textup{rank}(R^*)\cdot \|R\|_{\op}\le \tau k$. Therefore $R^*$ is feasible solution for the objective~\eqref{eqn:objective_relaxation} with $f(R^*)\le \epsilon$. 
	
	By Proposition~\ref{prop:lipschitz}, we have that $f(R)$ is 1-Lipschitz. Moreover, for any $R,S$ with $\|R\|_{S_1}\le \tau k, \|S\|_{S_1}\le \tau k$ we have that $ \|R-S\|_F \le  \|R\|_{F} + \|S\|_F\le \|R\|_{S_1}+ \|S\|_{S_1}\le 2\tau k$. Therefore the diameter of the constraint set is at most $\tau k$. 
	
	By~\cite[Theorem 4.4]{DBLP:conf/icml/HazanK12},
		we have that Frank-Wolfe algorithm returns solution $R$ with $f(R)-f(R^*)\le \epsilon+\delta$ in $\left(\frac{\tau k}{\delta}\right)^4$ iteration. 
\end{proof}

\section{Shorter codes with relaxed objective for Polynomial Spectral Components Analysis}\label{sec:shorter_code_appendix}

\paragraph{Notations}For a matrix $A$, let $\sigma_1(A)\ge \sigma_2(A)\ge ..$ be its singular values. Then the Schatten p-norm, denoted by $\|\cdot \|_{S_p}$, for $p\ge 1$ is defined as $\|A\|_{S_p} = \left(\sum_i \sigma_i(A)^p\right)^{1/p}$.  For even integer $p$, an equivalent and simple definition is that $\|A\|_{S_p}^p \triangleq \trace((A^{\top}A)^{p/2})$.

In this section we consider the following further relaxation of objective~\eqref{eqn:objective_relaxation}. 
\begin{align}
\textup{minimize}~~f_4(R) & := \Exp\left[\left\|Rx\tp{2} - x\tp{2}\right\|_{S_p}^2\right] \label{eqn:objective_relaxation_schatten}\\
\textup{s.t.}~& \norm{R}_{S_1} \le \tau k\nonumber
\end{align}
Since $\|A\|_F \ge \|A\|_{S_4}\ge \|A\|_{S_{\infty}} = \|A\|$, this is a relaxation of the objective~\eqref{eqn:objective_relaxation}, and it interpolates between kernal PCA and spectral decoding. Our assumption is weaker than kernal PCA but stronger than spectral decodable.  
\begin{definition}[Extension of definition~\ref{def:regularly_decodable}]\label{def:regularly_decodable_Sp}
	We say a data distribution $\D$ is $(k,\epsilon)$-\textit{regularly} spectral decodable with $\|\cdot\|_{S_p}$ norm if the error $E$ in equation~\eqref{eqn:1} is bounded by $\|E\|_{S_p}\le \epsilon$.  				\end{definition}

We can reduce the length of the code from $O(k^4)$ to $O(k^2)$ for any constant $p$. 
\begin{theorem}\label{thm:extension}
	Suppose data distribution is $(k,\epsilon)$-spectral decodable with norm $\|\cdot\|_{S_p}$ for $p = O(1)$, then solving ~\eqref{eqn:objective_relaxation_schatten} using (usual) Frank-Wolfe algorithm gives a solution $\hat{R}$ of 
		$k' = O(k^2\tau^2/\epsilon^2)$ with $f(R)\le \epsilon+ \delta$.  
	As a direct consequence, we obtain encoding and decoding pair $(g_A,h_B)\in \cH^{\spd}_{k'}$ with $k' = O(k^2\tau^2/\epsilon^2)$ and reconstruction error $\epsilon+\delta$. \end{theorem}

The main advantage of using relaxed objective is its smoothness. This allows us to optimize over the Schatten 1-norm constraint set much faster using usual Frank-Wolfe algorithm. Therefore the key here is to establish the smoothness of the objective function. Theorem~\ref{thm:extension} follows from the following proposition straightforwardly.

\begin{proposition}
	Objective function $f_p$ (equation~\eqref{eqn:objective_relaxation_schatten}) is convex and $O(p)$-smooth. 
\end{proposition}
\begin{proof}
	Since $\|\cdot\|_{S_p}$ is convex and composition of convex function with linear function gives convex function. Therefore, $\left\|Rx\tp{2} - x\tp{2}\right\|_{S_p}$ is a convex function. The square of an non-negative convex function is also convex, and therefore we proved that $f_p$ is convex. We prove the smoothness by first showing that $\|A\|_{S_p}^2$ is a smooth function with respect to $A$. We use the definition $\|A\|_{S_p}^p = \trace((A^{\top}A)^{p/2})$. Let $E$ be a small matrix that goes to 0, we have 
	\begin{align}
	\|A+E\|_{S_p}^p & = \trace((A^{\top}A)^{p/2}) + T_1 + T_2  + o(\|E\|_F^2)
	\end{align}
	where $T_1,T_2$ denote the first order term and second order term respectively. Let $U = A^{\top}E+E^{\top}A$ and $V = A^{\top}A$. We note that $T_2$ is a sum of the traces of matrices like $UVUVU^{p/2-2}$. By Lieb-Thirring inequality, we have that all these term can be bounded by $\trace(U^{p/2-2}V^2) = 2\trace((A^{\top}A)^{p/2-2}A^{\top}EE^{\top}A) + 2\trace((A^{\top}A)^{p/2-2}A^{\top}EA^{\top}E^{\top})$. 
	For the first term, we have that 
	\begin{align}
	\trace((A^{\top}A)^{p/2-2}A^{\top}EE^{\top}A)  & \le \|(AA^{\top})^{(p-2)/4}E\|^2 \le\|(AA^{\top})^{(p-2)/4}\|_{S_\infty}^2\|E\|_F^2 = \|A\|_{S_{\infty}}^{p-2}\|E\|_F^2\nonumber
	\end{align}
	where in the first inequality we use Cauchy-Schwarz. 
	Then for the second term we have
	\begin{align}
	\trace((A^{\top}A)^{p/2-2}A^{\top}EA^{\top}E^{\top}) &\le \|(A^{\top}A)^{(p-2)/4}E\|_F\|AE(A^{\top}A)^{(p-4)/4}\|_F \nonumber\\
	& \le  \|(A^{\top}A)^{(p-2)/4}E\|_F^2 \tag{by Lieb-Thirring inequality}\\
	&\le\|(AA^{\top})^{(p-2)/4}\|^2\|E\|_F^2 = \|A\|_{S_{\infty}}^{p-2}\|E\|_F^2\label{eqn:eqn13}
	\end{align}
	Therefore, finally we got 
	\begin{align}
		T_2 \le O(p^2) \cdot \|A\|_{S_{p-2}}^{p-2}\|E\|_F^2
	\end{align}
	Therefore, we complete the proof by having, 
	\begin{align}
	\|A+E\|_{S_p}^2& \le (\|A\|_{S_p}^p + T_1 + T_2 + o(\|E\|^2))^{2/p}\le \|A\|_{S_p}(1  + T_1'  +  \frac{2}{p\|A\|_{S_p}^{p/2}} T_2) + o(\|E\|^2) \tag{by $(1+x)^{p/2} \le 1+2x/p + o(\|x\|^2)$} \\
	& \le \|A\|_{S_p}^2 + T_1'' + O(p) \|E\|^2 + o(\|E\|^2)\tag{by equation~\eqref{eqn:eqn13}}
	\end{align}
	
	\end{proof}

\section{Toolbox}

\begin{lemma}\label{lem:sos_holder}
	Let $p\ge 2$ be a power of 2 and $u = [u_1,\dots, u_n]$ and $v = [v_1,\dots, v_n]$ be indeterminants. Then there exists SoS proof, 
		\begin{equation}
			\vdash \left(\sum_{j}u_iv_j^{p-1}\right)^p \le \left(\sum_i u_i^p\right)\left(\sum v_i^p\right)^{p-1}
		\end{equation}
\end{lemma}
\begin{proof}[Proof Sketch]
	The inequality follows from repeated application of Cauchy-Schwarz. For example, for $p = 4$ we have 
	\begin{align}
\vdash 	\left(\sum_i u_i^4\right)\left(\sum v_i^4\right)^{3} & \ge \left(\sum_i u_i^2v_i^2\right)^2\left(\sum v_i^4\right)^{2} \tag{by Cauchy-Schwarz}\\
& \ge \left(\sum_i u_iv_i^3\right)^4 \tag{by Cauchy-Schwarz again}
\end{align}
For $p = 2^s$ with $s >2$, the statement can be proved inductively.
\end{proof}

\end{document}